\documentclass[final,hidelinks,onefignum,onetabnum]{siamart251216}

\usepackage{lipsum}
\usepackage{amsfonts}
\usepackage{graphicx,subfigure}
\usepackage{tikz}
\usetikzlibrary{calc}
\usepackage{epstopdf}
\usepackage{algorithmic}
\usepackage{booktabs}
\usepackage{multirow} 
\usepackage{hyperref}
\usepackage{bm}
\usepackage[normalem]{ulem}
\ifpdf
  \DeclareGraphicsExtensions{.eps,.pdf,.png,.jpg}
\else
  \DeclareGraphicsExtensions{.eps}
\fi

\newcommand*\diff{\mathop{}\!\mathrm{d}}

\newsiamremark{remark}{Remark}
\newsiamremark{hypothesis}{Hypothesis}
\crefname{hypothesis}{Hypothesis}{Hypotheses}
\newsiamthm{claim}{Claim}
\newsiamremark{assumption}{Assumption}
\crefname{assumption}{Assumption}{Assumptions}

\headers{HRGrad}{Z.Y. Liang}

\title{Conflict-Aware Harmonized Rotational Gradient for Multiscale Kinetic Regimes
\thanks{Submitted to the editors DATE.}}

\author{Zhangyong Liang\thanks{National Center for Applied Mathematics, Tianjin University, Tianjin, 300072, China.}}

\usepackage{amsopn}

\makeatletter
\newcommand*{\addFileDependency}[1]{
  \typeout{(#1)}
  \@addtofilelist{#1}
  \IfFileExists{#1}{}{\typeout{No file #1.}}
}
\makeatother

\ifpdf
\hypersetup{
  pdftitle={HRGrad},
  pdfauthor={Z.Y. Liang}
}
\fi

\begin{document}

\maketitle

\begin{abstract}
     In this paper, we propose a harmonized rotational gradient method, termed HRGrad, for simultaneously tackling multiscale time-dependent kinetic problems with varying small parameters.
     These parameters exhibit asymptotic transitions from microscopic to macroscopic physics, making it a challenging multi-task problem to solve over all ranges simultaneously.
     Solving tasks in different asymptotic regions often encounter gradient conflicts, which can lead to the failure of multi-task learning.
     To address this challenge, we explicitly encode a hidden representation of these parameters, ensuring that the corresponding solving tasks are serialized for simultaneous training.
     Furthermore, to mitigate gradient conflicts, we segment the prediction results to construct task losses and introduce a novel gradient alignment metric to ensure a positive dot product between the final update and each loss-specific gradient.
     This metric maintains consistent optimization rates for all task losses and dynamically adjusts gradient magnitudes based on conflict levels.
     Moreover, we provide a mathematical proof demonstrating the convergence of the HRGrad method, which is evaluated across a range of challenging asymptotic-preserving neural networks (APNNs) scenarios.
     We conduct an extensive set of experiments encompassing the Bhatnagar-Gross-Krook (BGK) equation and the linear transport equation in all ranges of Knudsen number.
     Our results indicate that HRGrad effectively overcomes the `failure modes' of APNNs in these problems.
\end{abstract}



\section{Introduction}\label{sec:01}
Kinetic equations have been widely used in many areas such as mechanics, rarefied gas, plasma physics, astrophysics, semiconductor device modeling, and social and biological sciences \cite{Villani02}. They describe the non-equilibrium dynamics of a system composed of a large number of particles and bridge atomistic and continuum models in the hierarchy of multiscale modeling. The Boltzmann-type equation, as one of the most representative models in kinetic theory, provides a powerful tool to describe molecular gas dynamics, radiative transfer, plasma physics, and polymer flow \cite{A15}. In particular, the linear semiconductor Boltzmann-BGK equation we will study in this article describes the evolution of electron density(in the phase space) within semiconductor devices, in the presence of an external electrical potential and scatterings among various particles \cite{Jungel}.

Ignoring relativistic effects, quantum mechanics provides a sufficiently accurate framework for understanding the physical properties of matter. Within the broad category of multiscale modeling, kinetic equations play a central role as bridges between atomistic and continuum descriptions~\cite{cerci2002relat}. 
However, a fundamental challenge arises from their inherently multiscale nature: kinetic equations typically involve small or multiple spatial and/or temporal scales, characterized by the dimensionless Knudsen number, which represents the ratio of the mean free path (or time) to the macroscopic length (or time) scale.  
In multiscale computations, one often encounters the need to couple models across different scales, each requiring distinct numerical strategies~\cite{weinan2011principles}.
When the Knudsen number is small, kinetic equations can be rigorously approximated by macroscopic hydrodynamic or diffusion equations~\cite{bouchut2000kinetic}. 
Yet, direct numerical simulations in this regime become prohibitively expensive, since resolving the small scales explicitly is computationally intractable.  
To overcome this difficulty, the concept of asymptotic-preserving (AP) schemes was introduced. 
An AP scheme is designed to seamlessly capture the asymptotic transition from kinetic or hyperbolic equations to their macroscopic hydrodynamic or diffusive limits within the discrete numerical framework~\cite{jin2010asymptotic}. 
This strategy avoids the explicit coupling of microscopic and macroscopic solvers; instead, the microscopic scheme naturally degenerates into a macroscopic solver in the vanishing Knudsen number limit.  
Over the past two decades, AP schemes have proven to be a powerful and robust framework for multiscale problems. 
Their main advantage lies in their efficiency and accuracy in the hydrodynamic or diffusive regime, since they eliminate the need to resolve small physical parameters while faithfully capturing the emergent macroscopic behavior. 
As such, AP schemes provide a unifying and effective methodology for handling the inherent complexity of multiscale kinetic and hyperbolic problems. 

The high dimensionality and multiscale nature of kinetic simulations call for efficient numerical methods. 
Classical Monte Carlo methods are widely used but suffer from low-order accuracy and increasing cost as the Knudsen number decreases~\cite{pareschi1999implicit, dimarco2014numerical, pareschi2001time}. 
To alleviate these difficulties, a series of approaches based on deep neural networks (DNNs) has recently been proposed.  
Training these DNN-based solvers requires minimizing a high-dimensional, non-convex loss and remains difficult in the small-parameter regimes of multiscale kinetic problems. 
For time-dependent PDEs, physics-informed neural networks (PINNs) minimize least-squares residual risks of the governing equations~\cite{raissi2019physics}. 
Yet, depending on how the loss function is designed, such formulations are only able to capture the leading-order or single-scale behavior of the solution.  
Such limitations arise from loss design, first-order optimization, and the Frequency Principle~\cite{xu2019frequency}, which biases DNNs toward low frequencies and away from small-scale structures.  
To address these challenges, a new class of asymptotic-preserving neural networks (APNNs) \cite{jin2023ap,jin2024ap,liu2025ap,chen2025micro} has been developed, which integrates the asymptotic-preserving (AP) strategy with PINNs to tackle multiscale physical problems involving small-scale parameters. 
By embedding asymptotic information, APNNs overcome the shortcomings of standard PINNs and yield robust multiscale approximations.
However, existing APNNs are typically limited to problems with a specific small-scale parameter. 
For many multiscale physical problems, the scale parameter often spans the entire spectrum from macroscopic regimes, through transitional regions, to microscopic regimes. 
For example, in the radiative transfer equation, the scale parameter (i.e., the Knudsen number) may vary significantly, covering regimes from kinetic to diffusive, thus exhibiting pronounced multiscale transitions.
A natural question is whether APNNs can be pretrained across a range of scale parameters and directly predict the solution for downstream values, including those unseen during pretraining. 
Even with AP decomposition, however, heterogeneous scale-dependent behaviors can induce severe gradient conflicts across asymptotic regimes.

This raises two intertwined challenges: pretraining AP solvers across the parameter spectrum and resolving the resulting multi-task gradient conflicts. 
Existing methods for parameter-dependent PDEs either rely on data-intensive neural operators or use collocation-based solvers that typically require retraining for new parameter values. 
To address this limitation, two directions have been explored: multi-task learning (MTL)~\cite{caruana1997multitask,zhang2018overview,zhang2021survey,standley2020tasks} over a parameter domain, and parameterization strategies for specific parameter values. 
Existing approaches to conflicting gradients in MTL mainly focus on dynamic loss re-weighting in classical machine learning \cite{vandenhende2021multi,yu2020gradient}, with only limited extensions to PDE solvers, where residual, initial, and boundary losses may be imbalanced. 
In this setting, self-adaptive weighting schemes \cite{song2024loss, van2022optimally,wang2022and,xiang2022self,yao2023multiadam} are widely used, while self-paced learning (SPL) \cite{jiang2015self,yuan2024self,zhao2024symmetric} further emphasizes difficult collocation-point losses; Cognitive Physics-Informed Neural Networks (CoPINN) \cite{duancopinn} recast weighted-loss optimization as collocation-point selection. 
Another line of work addresses imbalance through optimizer preconditioning. 
Unlike self-adaptive weighting schemes that focus on gradient magnitudes, these methods modify gradient directions in MTL, where opposing gradients can yield inefficient updates. 
Representative methods include PCGrad \cite{yu2020gradient}, CAGrad \cite{liu2021conflict}, IMTL-G \cite{liu2021towards}, ConFIG \cite{liu2024config}, and SOAP \cite{wang2025gradient}. 
However, these approaches mainly address PDEs with fixed parameters. 
For simultaneous pretraining over varying parameter values, recent work has explored meta-learning \cite{de2021hyperpinn,huang2022meta,toloub2023dats}, typically via a bi-level formulation in which the outer loop assigns weights across parameter values and the inner loop solves the associated PINN problems. 
Another approach directly parameterizes the scale parameter itself \cite{park2024parameterized}, feeding sampled parameter values to the network as an additional input while encoding collocation points and parameters separately before concatenation. 
Such parameterization, however, is less effective for strongly multiscale problems with drastic regime changes and severe gradient conflicts. 
To the best of our knowledge, no existing optimization framework resolves these conflicts for simultaneous training across asymptotically distinct parameter regimes.

Motivated by these challenges, we develop a physics-informed continuous multi-task optimization framework for asymptotic-preserving neural networks over a continuum of scale parameters. 
Our central observation, which is formalized in \Cref{sec:method}, is that multiscale APNN training suffers from two coupled failure modes: directional contradiction between microscopic and macroscopic task gradients, and magnitude disparity caused by stiff asymptotic scaling. 
These effects make standard projection-based gradient manipulation intrinsically dissipative for multiscale kinetic learning, as it clips weak microscopic gradients precisely when non-equilibrium features are the most vulnerable. 
To overcome this difficulty, we propose the Harmonized Rotational Gradient (HRGrad) method, which replaces lossy Euclidean projection with harmonized, energy-preserving gradient rotation and fair aggregation. As illustrated in \Cref{fig:hrgrad_framework}, HRGrad proceeds through four coordinated stages: harmonized-cone construction with physical anchoring, isometric task-wise rotation, MER-based angle selection, and fair consensus aggregation. The first three stages resolve directional contradiction without dissipating gradient energy. The final aggregation stage mitigates magnitude disparity by restoring balanced multiscale contributions in the resulting update.

\begin{figure*}[t]
    \centering
    \includegraphics[width=0.92\textwidth]{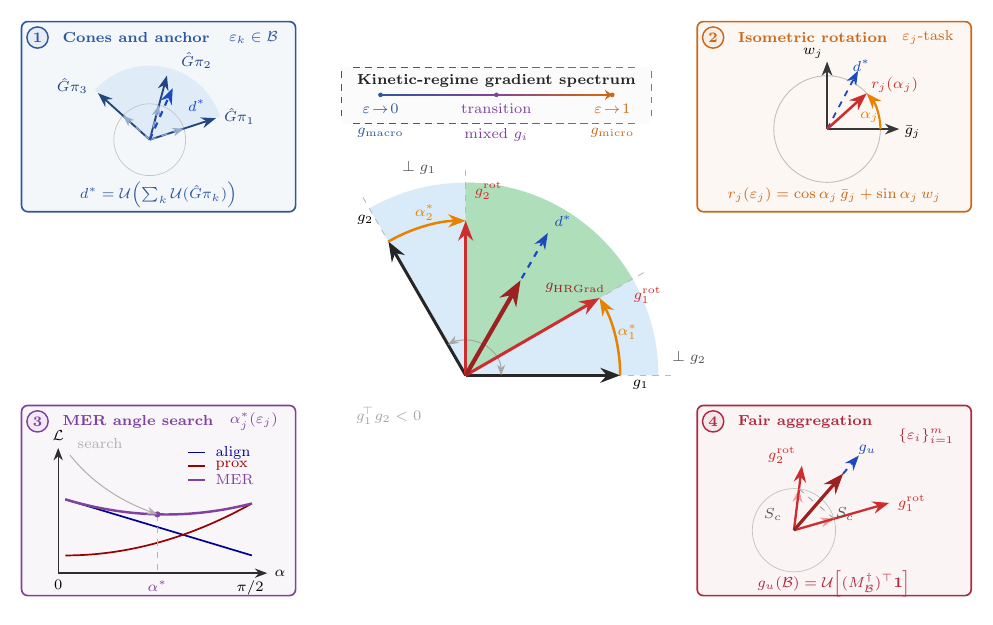}
    \caption{Schematic of the proposed HRGrad framework. Starting from task gradients sampled over macro, transitional, and micro regimes, HRGrad first constructs a harmonized cone together with a physical anchor. Then, it rotates conflicting gradients by isometric operators and selects MER-optimal rotation angles. Finally, it performs fair aggregation to produce the update direction.}
    \label{fig:hrgrad_framework}
\end{figure*}

Our main contributions can be summarized as follows:
\begin{itemize}
    \item \textbf{HRGrad for multiscale APNNs.} We propose the Harmonized Rotational Gradient (HRGrad) method as a physics-informed gradient preprocessing and aggregation framework for continuous multi-task APNN training over a continuum of Knudsen numbers, enabling one unified model to learn across kinetic, transitional, and macroscopic regimes.
    \item \textbf{Kinetic-aware gradient guidance.} We design HRGrad to explicitly account for the multiscale structure of kinetic equations, so that weak but essential updates associated with microscopic non-equilibrium dynamics are protected during APNN optimization rather than being overwhelmed by dominant macroscopic gradients.
    \item \textbf{Harmonized rotational preprocessing.} We introduce a harmonized cone together with a physically anchored reference direction, and use manifold isometric rotation to preprocess conflicting task gradients. This mechanism removes destructive gradient conflicts while preserving the norm information carried by the original gradients, especially in the kinetic regime.
    \item \textbf{Magnitude restoration and fair aggregation.} We further equip HRGrad with magnitude restoration and fair consensus aggregation, so that the final update remains positively aligned with all rotated task gradients and balances optimization progress across different scales and physical objectives.
    \item \textbf{Theory and multiscale validation.} We provide a convergence analysis for the proposed gradient preprocessing framework and validate HRGrad on the Boltzmann--BGK equation, the linear transport equation, the ES-BGK equation, and the linear semiconductor Boltzmann--Poisson equation, demonstrating stable and robust performance across diverse multiscale kinetic regimes.
\end{itemize}

The remainder of the paper is organized as follows. \Cref{sec:problem} introduces the Boltzmann-BGK equation and the linear transport equation, commonly used asymptotic-preserving numerical methods, and the formulation of the APNNs loss.
\Cref{sec:method} analyzes the failure modes of multiscale APNN training and develops the HRGrad method, including its geometric motivation, theoretical properties, and implementation.
\Cref{sec:result} showcases numerical results for several kinetic equations solved by HRGrad and compares them with existing multi-task learning baselines and parameterized APNN strategies.
Finally, \Cref{sec:conclusion} offers the conclusion and outlook.

\section{Preliminaries}\label{sec:problem}
This section briefly introduces the kinetic equations of interest and the framework of asymptotic-preserving neural networks (APNNs).

\subsection{Multiscale kinetic problems}
We consider two prototypical multiscale kinetic equations that describe the evolution of a particle distribution function $f(t,\bm{x},\bm{v})$ under the scaling of the Knudsen number $\varepsilon$.

\paragraph{The Boltzmann-BGK equation}
The Boltzmann-BGK equation reads
\begin{equation}\label{eqn:bgk}
  \partial_t f + v \partial_x f = \frac{1}{\varepsilon} \left( M(U) - f \right),
\end{equation}
where $M(U)$ is the local Maxwellian defined by the macroscopic moments $U = \langle m f \rangle = (\rho, \rho u, E)^T$, with $m=(1,v,\frac{1}{2}v^2)^T$ and $\langle \cdot \rangle = \int \cdot \, d\bm{v}$. 
As $\varepsilon \to 0$, we have $f \to M(U)$, and the moments satisfies the local conservation laws:
\begin{equation}\label{eqn:claw}
  \partial_t U + \partial_x \langle v m f \rangle = 0,
\end{equation}
which reduces to the macroscopic compressible Euler equations \cite{filbet2010class}.

\paragraph{The linear transport equation}
The one-dimensional linear transport equation under diffusive scaling is 
\begin{equation}\label{eqn:lt}
    \varepsilon \partial_t f + v \partial_x f = \frac{1}{\varepsilon} \left( \frac{1}{2}\int_{-1}^1 f \, d\bm{v}' - f \right).
\end{equation}
By defining the even and odd parities $r(t,\bm{x},\bm{v}) = \frac{1}{2}(f(\bm{v})+f(-\bm{v}))$ and $j(t,\bm{x},\bm{v}) = \frac{1}{2\varepsilon}(f(\bm{v})-f(-\bm{v}))$, the system transforms into:
\begin{equation}\label{eqn:even-odd-system}
    \left\{
    \begin{aligned}
        \varepsilon^2 \partial_t r + \varepsilon^2 v\partial_x j &= \rho - r, \\
        \varepsilon^2 \partial_t j + v\partial_x r &= - j,
    \end{aligned}
    \right.
\end{equation}
where $\rho = \langle r \rangle := \int_0^1 r\,d\bm{v}$. In the diffusive limit $\varepsilon \to 0$, we obtain $r \to \rho$ and $j \to -v\partial_x \rho$, and $\rho$ satisfies the diffusion equation:
\begin{equation}\label{eqn:diff}
  \partial_t\rho - \frac13\,\partial_{xx}\rho = 0.
\end{equation}

\subsection{Asymptotic-preserving neural networks (APNNs)}
Standard Physics-Informed Neural Networks (PINNs) approximate $f$ by directly minimizing the PDE residuals. However, in the diffusive or fluid regime ($\varepsilon \to 0$), the standard PINN loss fundamentally fails to capture the macroscopic evolution constraints, leading to incorrect predictions of the macroscopic dynamics \cite{jin2023ap,jin2024ap,liu2025ap}.

To overcome this, Asymptotic-Preserving Neural Networks (APNNs) \cite{jin2023ap} modify the loss function by explicitly coupling the microscopic kinetic equations with their corresponding macroscopic conservation laws. 

\begin{figure}[ht]
    \centering
    \includegraphics[width=0.45\textwidth]{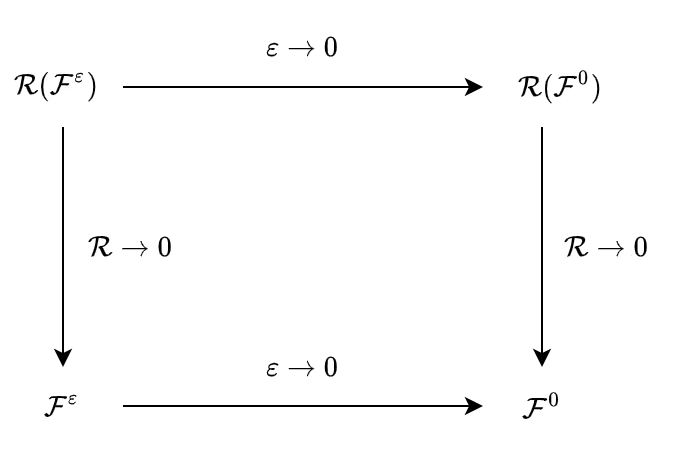}
    \vspace{-16pt}
    \caption{Illustration of the APNNs framework.}
    \label{fig:apnns}
\end{figure}

As shown in \Cref{fig:apnns}, the strategy involves embedding the macroscopic equations directly into the neural network training. 
For instance, when solving the Boltzmann-BGK equation, APNNs approximate $f$ and $U$ by continuous neural networks $f_\theta$ and $U_\theta$, and define the loss functional as
\begin{equation}\label{eqn:loss-bgk}
  \mathcal{L}^{\varepsilon}_{\mathrm{APNNs}}
  = \mathcal{L}^{\varepsilon}_{\mathrm{Res}}
  +\mathcal{L}^{\varepsilon}_{\mathrm{Claw}}
  +\mathcal{L}^{\varepsilon}_{\mathrm{Cnt}}
  +\mathcal{L}^{\varepsilon}_{\mathrm{BC}}
  +\mathcal{L}^{\varepsilon}_{\mathrm{IC}},
\end{equation}
where the PDE structure constraints are:
\begin{equation}\label{eqn:loss-bgk-terms}
\left\{
\begin{aligned}
  \mathcal{L}^{\varepsilon}_{\mathrm{Res}}
  &=\lambda_1\bigl\|\varepsilon(\partial_t f_\theta+v\,\partial_x f_\theta)-(M(U_\theta)-f_\theta)\bigr\|^2,\\
  \mathcal{L}^{\varepsilon}_{\mathrm{Claw}}
  &=\lambda_2\bigl\|\partial_t U_\theta+\partial_x\langle v m f_\theta\rangle\bigr\|^2,\\
  \mathcal{L}^{\varepsilon}_{\mathrm{Cnt}}
  &=\lambda_3\bigl\|U_\theta-\langle m f_\theta\rangle\bigr\|^2.
\end{aligned}
\right.
\end{equation}
Similarly, for the linear transport equation, APNNs enforce the macroscopic density evolution:
\begin{equation}\label{eqn:moment}
  \partial_t \rho_\theta + \langle v \partial_x j_\theta \rangle = 0.
\end{equation}
By jointly minimizing the residuals of both the microscopic and macroscopic equations, APNNs preserve the correct asymptotic behavior as $\varepsilon \to 0$, enabling uniformly accurate simulations across all scales of $\varepsilon$.

\section{Methodology}\label{sec:method}

To simultaneously solve the multiscale kinetic equations across the continuous Knudsen spectrum $\varepsilon \in (0, 1]$, we formulate the training process as a continuous multi-task learning (MTL) problem. In each training iteration, we sample a mini-batch of $m$ scale parameters $\mathcal{B} = \{\varepsilon_1, \varepsilon_2, \cdots, \varepsilon_m\}$. For each sampled parameter $\varepsilon_i$, we adopt the Asymptotic-Preserving Neural Network (APNN) architecture described in Section~\ref{sec:problem} as the unified base model. Consequently, the objective for the $i$-th task is precisely defined as the APNN loss functional evaluated at $\varepsilon_i$, i.e., $\mathcal{L}_i(\theta) := \mathcal{L}_{\mathrm{APNNs}}^{\varepsilon_i}(\theta)$ (as formulated in Eq.~\ref{eqn:loss-bgk-terms} for the Boltzmann-BGK equation or its equivalent for the linear transport equation). The optimization trajectory is thus driven by the task-specific gradients $g_i = \nabla_\theta \mathcal{L}_i(\theta)$.

When the neural solver is faced with tasks across drastically different kinetic regimes simultaneously, we observe that existing MTL methods suffer from severe `failure modes', making multi-task optimization exceptionally challenging. To clarify the rationale and innovation behind our proposed method, we first mathematically dissect these failure modes from the perspective of multiscale gradient conflicts. Based on this analysis, we develop the Harmonized Rotational Gradient (HRGrad) method, which abandons traditional dissipative projections in favor of manifold isometric rotations, ensuring that the preconditioning gradient consistently maintains physically valid, conflict-free, and energy-preserving updates.

\subsection{Gradient conflict in multiscale kinetic regimes}
Although APNNs employ macro-micro and even-odd decompositions \cite{jin2023ap,jin2024ap,liu2025ap} to enforce asymptotic correctness within a \textit{single} regime, jointly training tasks across vastly different $\varepsilon$ values with shared parameters $\theta$ inevitably triggers severe gradient conflicts (i.e., $\langle g_i, g_j \rangle < 0$). This geometric misalignment originates from the fundamentally opposing physical behaviors captured by the APNN sub-losses across different kinetic regimes. Specifically, let $g_{\text{micro}} = \nabla_\theta \mathcal{L}_{\mathrm{Res}}^\varepsilon$ denote the gradient driven by microscopic transport mechanisms, and $g_{\text{macro}} = \nabla_\theta \mathcal{L}_{\mathrm{Claw}}^\varepsilon$ (or macroscopic density evolution) denote the gradient driven by macroscopic relaxation and conservation. We analyze the dynamics across three typical regimes:

\begin{itemize}
  \item \textbf{Micro-limit regime ($\varepsilon \to 1$; kinetic side):}
  The dynamics are highly non-equilibrium and transport-dominated. The optimization is overwhelmingly guided by the microscopic residual $\mathcal{L}_{\mathrm{Res}}^\varepsilon$ (e.g., the transport term $\varepsilon(\partial_t f_\theta + v\partial_x f_\theta)$ in Eq.~\ref{eqn:loss-bgk-terms}). The corresponding gradients $g_{\text{micro}}$ actively push the network to capture and preserve high-frequency, anisotropic structures in the phase space (such as Knudsen boundary layers), strongly resisting any premature relaxation toward the local Maxwellian $M(U_\theta)$.
  
  \item \textbf{Macro-limit regime ($\varepsilon \to 0$; fluid/diffusive side):}
  The system reaches the hydrodynamic or diffusive limit. The stiff collision penalty $\frac{1}{\varepsilon}(M(U_\theta) - f_\theta)$ strictly enforces the local relaxation $f_\theta \to M(U_\theta)$ (or $r_\theta \to \rho_\theta$ for the linear transport equation). The optimization focus shifts entirely to the macroscopic conservation laws (e.g., $\mathcal{L}_{\mathrm{Claw}}^\varepsilon$ in Eq.~\ref{eqn:loss-bgk-terms} or the density evolution in Eq.~\ref{eqn:moment}). The gradients $g_{\text{macro}}$ from this regime drive the neural network parameters to rapidly collapse onto a smooth, low-frequency macroscopic equilibrium manifold, heavily penalizing deviations from local equilibrium.
  
  \item \textbf{Intermediate regime ($\varepsilon \in (\varepsilon_{\min}, 1)$):} 
  Neither physical limit dominates; free transport and collisional relaxation contribute comparably. The network receives mixed gradient signals demanding both macroscopic consistency (moment evolution) and microscopic fidelity (non-equilibrium detail). Mini-batches from this band yield gradients that fluctuate between macro-scale isotropic smoothing and micro-scale anisotropic retention, severely increasing the frequency of pairwise disagreements.
\end{itemize}

From a mathematical optimization perspective, this multiscale physical transition introduces two fatal ``failure modes'' for conventional gradient aggregation:

\paragraph{\textbf{Failure Mode 1: Directional Contradiction ($\langle g_{\text{micro}}, g_{\text{macro}} \rangle < 0$)}} 
Macroscopic objectives actively smooth out non-equilibrium details to enforce local equilibrium, while microscopic objectives strive to preserve them. A parameter update step $\Delta \theta = -\eta g_{\text{macro}}$ that strictly favors the macroscopic continuum limit will inevitably degrade the microscopic transport fidelity. Based on the first-order Taylor expansion, the variation in the microscopic loss is:
\begin{equation}
    \mathcal{L}_{\text{micro}}(\theta + \Delta \theta) \approx \mathcal{L}_{\text{micro}}(\theta) - \eta \langle g_{\text{micro}}, g_{\text{macro}} \rangle.
\end{equation}
Since the underlying physical mechanisms are fundamentally opposing (isotropic smoothing vs. anisotropic high-frequency preservation), we rigorously have $\langle g_{\text{micro}}, g_{\text{macro}} \rangle < 0$. This yields an undesirable strict increase in the micro-loss ($\mathcal{L}_{\text{micro}}(\theta + \Delta \theta) > \mathcal{L}_{\text{micro}}(\theta)$). Consequently, these gradients inherently point in geometrically opposing directions, yielding persistent obtuse angles $\phi \in (\pi/2, \pi]$ and triggering severe negative transfer across the continuous $\varepsilon$ spectrum.

\paragraph{\textbf{Failure Mode 2: Magnitude Disparity and Energy Clipping}} 
Beyond directional conflict, the explicit $\mathcal{O}(1/\varepsilon)$ (in Eq.~\ref{eqn:bgk}) or $\mathcal{O}(1/\varepsilon^2)$ (in Eq.~\ref{eqn:even-odd-system}) scaling introduces extreme numerical kinetic stiffness. As $\varepsilon \to 0$, the macroscopic conservation and relaxation losses exponentially dominate the total loss landscape, creating an extreme spectral gap in gradient magnitudes:
\begin{equation}
    \lim_{\varepsilon \to 0} \frac{\|g_{\text{macro}}\|_2}{\|g_{\text{micro}}\|_2} = \mathcal{O}\left(\frac{1}{\varepsilon^p}\right) \to \infty, \quad (p \ge 1).
\end{equation}
This mathematically guarantees that the macroscopic gradient overwhelmingly eclipses the microscopic gradient ($\|g_{\text{macro}}\|_2 \gg \|g_{\text{micro}}\|_2$). 

Under this extreme multiscale heterogeneity, standard gradient manipulation methods fundamentally fail. Traditional optimizers (e.g., PCGrad, ConFIG) rely on Euclidean orthogonal projections $\mathcal{P}$ to resolve conflicts, blindly projecting the conflicting components of the microscopic gradients onto the normal plane of the dominant macroscopic gradient $g_{\text{macro}}$:
\begin{equation}
    \mathcal{P}(g_{\text{micro}}) = g_{\text{micro}} - \frac{\langle g_{\text{micro}}, g_{\text{macro}} \rangle}{\|g_{\text{macro}}\|_2^2} g_{\text{macro}}.
\end{equation}
By evaluating the remaining kinetic energy ($L_2$-norm) of the projected microscopic gradient, we reveal the inherent mathematical flaw of the projection operator:
\begin{equation}
\begin{aligned}
    \|\mathcal{P}(g_{\text{micro}})\|_2^2 &= \|g_{\text{micro}}\|_2^2 \left( 1 - \frac{\langle g_{\text{micro}}, g_{\text{macro}} \rangle^2}{\|g_{\text{micro}}\|_2^2 \|g_{\text{macro}}\|_2^2} \right) \\
    &= \|g_{\text{micro}}\|_2^2 (1 - \cos^2 \phi) \\
    &= \|g_{\text{micro}}\|_2^2 \sin^2 \phi.
\end{aligned}
\end{equation}
    This equation uncovers the catastrophic \textit{energy clipping} mechanism: when the multiscale conflict is severe ($\phi \to \pi$), $\sin^2 \phi \to 0$. Because the original microscopic kinetic signal $\|g_{\text{micro}}\|_2$ is already extremely weak due to the magnitude disparity, the projection mathematically acts as a \textit{numerical low-pass filter}, irreversibly diminishing its magnitude to near zero ($\|\mathcal{P}(g_{\text{micro}})\|_2 \to 0$). As illustrated in \Cref{fig:motivation_failure}(a), traditional projection methods (e.g., PCGrad, ConFIG) mathematically act as an energy clipping mechanism ($\|\mathcal{P}(g_{\text{micro}})\|_2 \ll \|g_{\text{micro}}\|_2$). As the conflict angle $\phi \to 180^\circ$, the norm $\|\mathcal{P}(g_{\text{micro}})\|_2 \to 0$, irreversibly dissipating high-frequency kinetic features. To address this, our proposed Harmonized Rotational Gradient (HRGrad) performs an isometric manifold rotation towards a physical anchor $d^*$ (\Cref{fig:motivation_failure}(b)). This approach rigorously guarantees non-conflict while flawlessly preserving $100\%$ of the gradient magnitude (i.e., kinetic physical energy, $\|g_{\text{micro}}^{\text{rot}}\|_2 \equiv \|g_{\text{micro}}\|_2$).

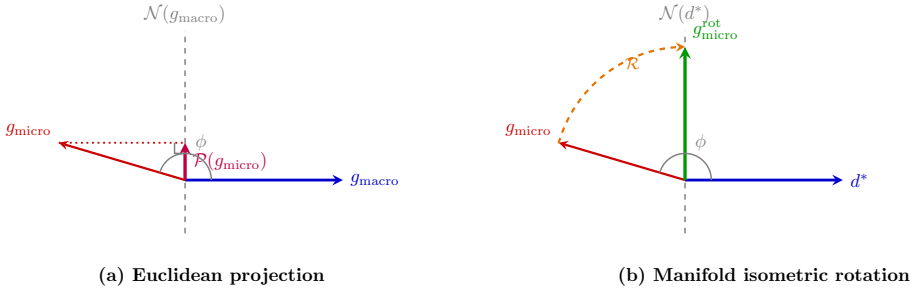
\begin{figure}[htbp]
    \centering
    \resizebox{0.95\textwidth}{!}{
    \begin{tikzpicture}[>=stealth, thick]
        \begin{scope}[xshift=0cm]
            \draw[->, line width=1.5pt, blue!80!black] (0,0) -- (3, 0) node[right] {$g_{\text{macro}}$};
            
            \draw[dashed, gray] (0, -1.0) -- (0, 2.8) node[above] {$\mathcal{N}(g_{\text{macro}})$};
            
            \coordinate (Gmicro) at (-2.4, 0.7);
            \draw[->, line width=1.2pt, red!80!black] (0,0) -- (Gmicro) node[above left] {$g_{\text{micro}}$};
            
            \draw[dotted, red!80!black, line width=1pt] (Gmicro) -- (0, 0.7);
            
            \draw[->, line width=1.8pt, purple] (0,0) -- (0, 0.7) node[midway, right] {$\mathcal{P}(g_{\text{micro}})$};
            
            \draw [gray, domain=0:164] plot ({0.5*cos(\x)}, {0.5*sin(\x)});
            \node[gray] at (0.3, 0.7) {$\phi$};
            
            \draw[gray] (-0.2, 0.7) -- (-0.2, 0.5) -- (0, 0.5);
            
            \node at (0.5, -1.8) {\textbf{(a) Euclidean projection}};
        \end{scope}
        
        \begin{scope}[xshift=9.5cm]
            \draw[->, line width=1.5pt, blue!80!black] (0,0) -- (3, 0) node[right] {$d^*$};
            
            \draw[dashed, gray] (0, -1.0) -- (0, 2.8) node[above] {$\mathcal{N}(d^*)$};
            
            \coordinate (Gmicro2) at (-2.4, 0.7);
            \draw[->, line width=1.2pt, red!80!black] (0,0) -- (Gmicro2) node[above left] {$g_{\text{micro}}$};
            
            \draw[->, dashed, orange!90!black, line width=1.2pt, domain=164:90] plot ({2.5*cos(\x)}, {2.5*sin(\x)});
            \node[orange!90!black, font=\small] at (-1.0, 2.2) {$\mathcal{R}$};
            
            \coordinate (Grot) at (0, 2.5);
            \draw[->, line width=1.8pt, green!60!black] (0,0) -- (Grot) node[above right] {$g_{\text{micro}}^{\text{rot}}$};
            
            \draw [gray, domain=0:164] plot ({0.5*cos(\x)}, {0.5*sin(\x)});
            \node[gray] at (0.3, 0.7) {$\phi$};
            
            \node at (1.5, -1.8) {\textbf{(b) Manifold isometric rotation}};
        \end{scope}
    \end{tikzpicture}
    }
    \caption{Comparison of gradient aggregation strategies under multiscale conflicts. (a) Euclidean projection severely clips the magnitude of the microscopic gradient $\mathcal{P}(g_{\text{micro}})$. (b) Manifold isometric rotation avoids directional conflicts while preserving the gradient magnitude $\|g_{\text{micro}}^{\text{rot}}\|_2 \equiv \|g_{\text{micro}}\|_2$.}
    \label{fig:motivation_failure}
\end{figure}

As a result, the high-frequency non-equilibrium physical features are permanently erased during parameter updates. The shared APNNs suffer from \textit{asymptotic manifold drift} and degenerate into a naive macroscopic Euler/Diffusion solver, entirely losing its capacity to capture multiscale transition details. This profound vulnerability mathematically necessitates the development of the Harmonized Rotational Gradient (HRGrad) method, which abandons dissipative projections in favor of isometric manifold rotations.

\begin{figure}[!htb]
     \centering
     \subfigure[$\varepsilon=1.0$]{\includegraphics[width=\textwidth]{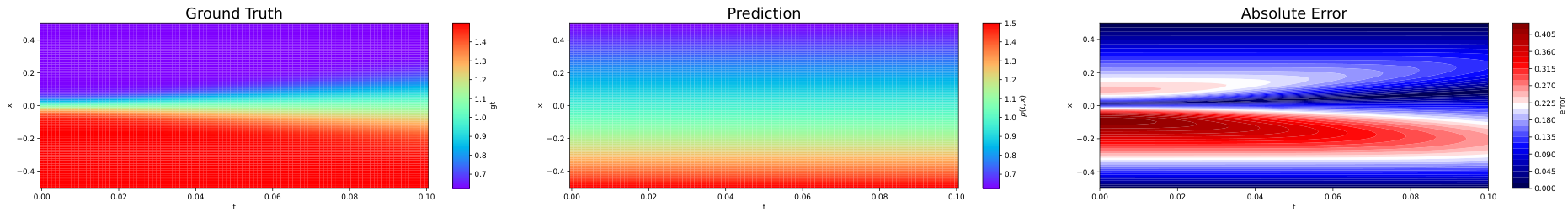}} \\
     \vspace{-8pt}
     \subfigure[$\varepsilon=0.1$]{\includegraphics[width=\textwidth]{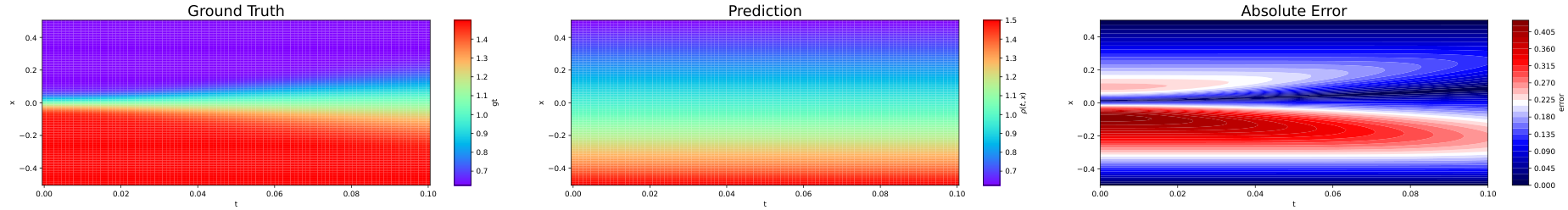}} \\
     \vspace{-8pt}
     \subfigure[$\varepsilon=0.01$]{\includegraphics[width=\textwidth]{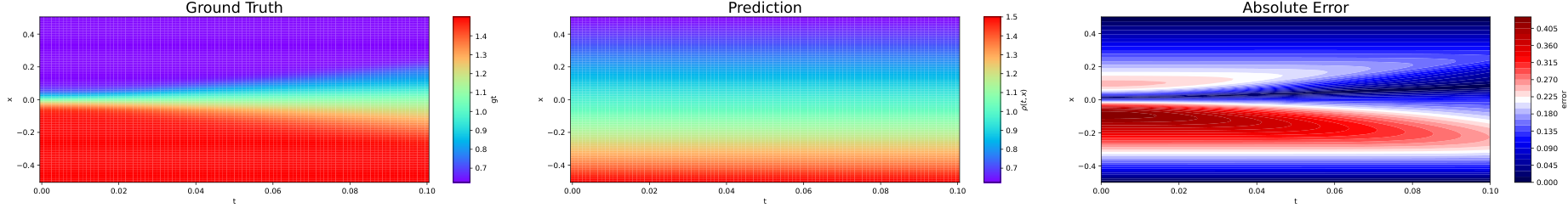}} \\
     \vspace{-8pt}
     \caption{MAD-MTL predicts the macroscopic density moments $\rho$ of the 1D Riemann problem. The comparison results are shown for (a) $\varepsilon = 1.0$, (b) $\varepsilon = 0.1$, and (c) $\varepsilon = 0.01$.}
\end{figure}

\begin{figure}[!htb]
     \centering
     \subfigure[$\varepsilon=1.0$]{\includegraphics[width=\textwidth]{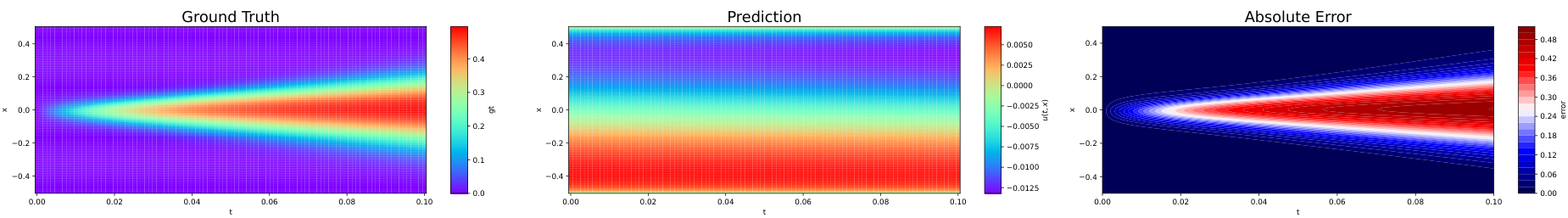}} \\
     \vspace{-8pt}
     \subfigure[$\varepsilon=0.1$]{\includegraphics[width=\textwidth]{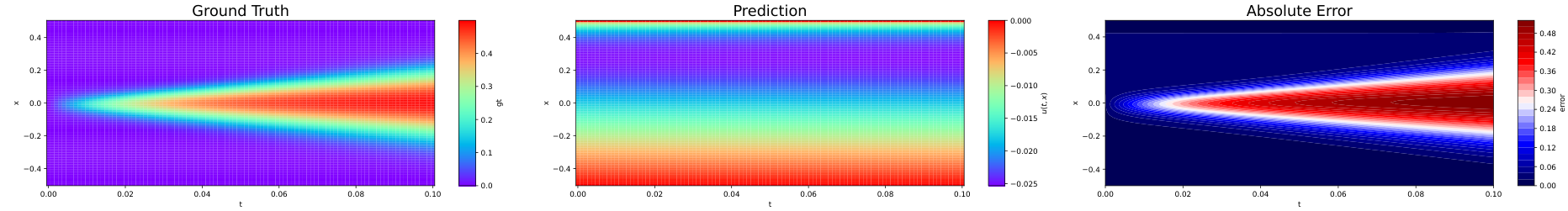}} \\
     \vspace{-8pt}
     \subfigure[$\varepsilon=0.01$]{\includegraphics[width=\textwidth]{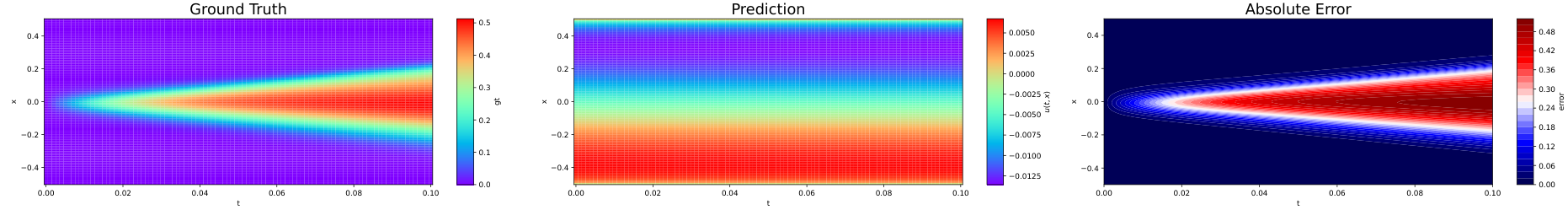}} \\
     \vspace{-8pt}
     \caption{MAD-MTL predicts the macroscopic velocity moments $u$ of the 1D Riemann problem. The comparison results are shown for (a) $\varepsilon = 1.0$, (b) $\varepsilon = 0.1$, and (c) $\varepsilon = 0.01$.}
\end{figure}

\begin{figure}[!htb]
     \centering
     \subfigure[$\varepsilon=1.0$]{\includegraphics[width=\textwidth]{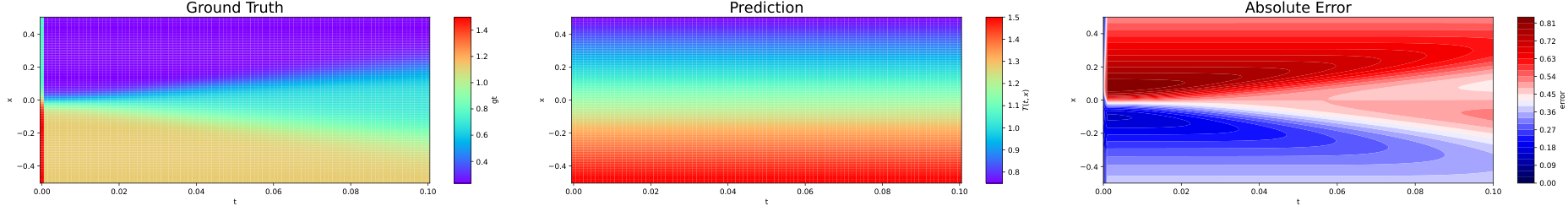}} \\
     \vspace{-8pt}
     \subfigure[$\varepsilon=0.1$]{\includegraphics[width=\textwidth]{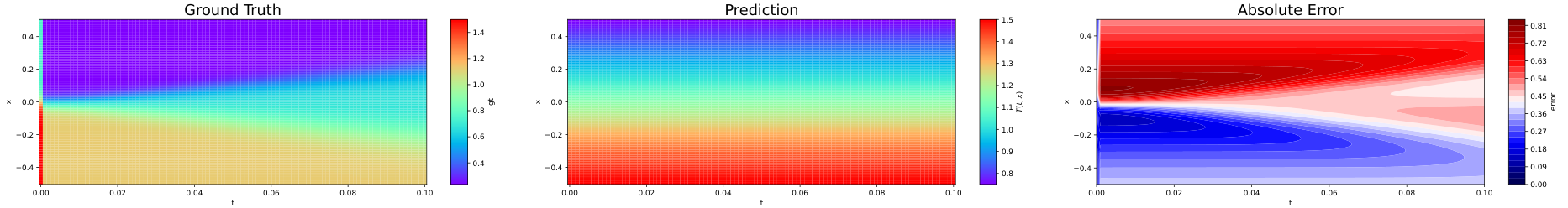}} \\
     \vspace{-8pt}
     \subfigure[$\varepsilon=0.01$]{\includegraphics[width=\textwidth]{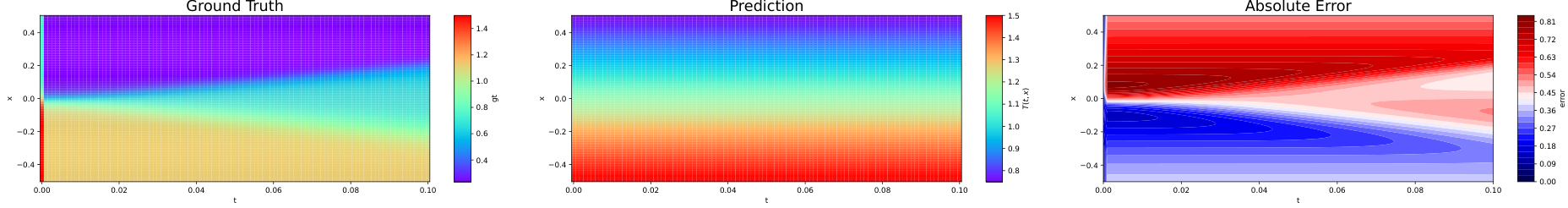}} \\
     \vspace{-8pt}
     \caption{MAD-MTL predicts the macroscopic temperature moments $T$ of the 1D Riemann problem. The comparison results are shown for (a) $\varepsilon = 1.0$, (b) $\varepsilon = 0.1$, and (c) $\varepsilon = 0.01$.}
\end{figure}

\subsection{Harmonized Rotational Gradient (HRGrad)}

Generally, we consider a multiscale optimization procedure with a set of $m$ task losses, i.e., $\{\mathcal{L}_1,\mathcal{L}_2,\cdots,\mathcal{L}_m\}$, sampled from a continuous Knudsen spectrum $\varepsilon \in (0,1]$.
Let $G = [g_1,g_2, \cdots, g_m] \in \mathbb{R}^{D \times m}$ denote the primal gradient matrix corresponding to each loss function, where $g_i = \nabla_\theta \mathcal{L}_i$. 

Traditional gradient manipulation methods (e.g., PCGrad, ConFIG) rely on Euclidean orthogonal projections to resolve gradient conflicts (i.e., $g_i^\top g_j < 0$). However, orthogonal projection $\mathcal{P}$ is inherently a non-isometric mapping, meaning that its operation strictly diminishes the gradient norm ($\|\mathcal{P}(g_i)\| < \|g_i\|$). In kinetic theory, microscopic non-equilibrium mechanisms (e.g., Knudsen boundary layers) produce gradients with significantly smaller magnitudes than macroscopic conservation laws. As intuitively illustrated in \Cref{fig:method_comparison}(a), PCGrad resolves conflicts by summing orthogonal projections ($\mathcal{O}$), which acts as a \textit{mathematical energy clipping} operator ($\|\mathcal{O}(g_i,g_j)\| \ll \|g_i\|$), inadvertently dissipating these high-frequency microscopic physical features. IMTL-G (\Cref{fig:method_comparison}(b)) rescales gradients to be equal-length but may not fully guarantee manifold feasibility. ConFIG (\Cref{fig:method_comparison}(c)) attempts to mitigate magnitude loss by summing unit vectors $\mathcal{U}$ of the orthogonal components, but the underlying direction still relies on a lossy projection basis. Furthermore, blindly seeking non-conflict directions in Euclidean space often forces the update trajectory to drift away from the physically feasible asymptotic-preserving (AP) manifold.

\begin{figure*}[htbp]
    \centering
    \resizebox{\textwidth}{!}{
    \begin{tikzpicture}[>=stealth, thick, scale=0.95, every node/.style={scale=0.95}]
        
        \begin{scope}[xshift=0cm]
            \coordinate (O) at (0,0);
            \coordinate (g1) at (0:2.5);
            \coordinate (g2) at (120:2.5);
            
            \coordinate (Og2) at (90:2.165); 
            \coordinate (Og1) at (30:2.165); 
            
            \draw[->, line width=1.5pt, black!70] (O) -- (g1) node[right] {$g_1$};
            \draw[->, line width=1.5pt, black!70] (O) -- (g2) node[above left] {$g_2$};
            
            \draw[->, line width=1.5pt, teal!80!black] (O) -- (Og2) node[above left] {$\mathcal{O}(g_2, g_1)$};
            \draw[->, line width=1.5pt, teal!80!black] (O) -- (Og1) node[right] {$\mathcal{O}(g_1, g_2)$};
            
            \draw[dashed, gray] (g2) -- (Og2);
            \draw[dashed, gray] (g1) -- (Og1);
            
            \draw[gray] (0, 1.965) -- (-0.2, 1.965) -- (-0.2, 2.165);
            \begin{scope}[shift={(Og1)}, rotate=30]
                \draw[gray] (0, -0.2) -- (-0.2, -0.2) -- (-0.2, 0);
            \end{scope}
            
            \coordinate (gPC) at ($(Og1)+(Og2)$);
            \draw[->, line width=1.8pt, green!60!black] (O) -- (gPC) node[above] {$g_{\text{PCGrad}}$};
            \draw[dashed, gray] (Og2) -- (gPC) -- (Og1);
            
            \node[font=\large\bfseries] at (1.25, -1.2) {(a) PCGrad};
        \end{scope}

        \begin{scope}[xshift=5.5cm]
            \coordinate (O) at (0,0);
            \coordinate (g1) at (0:2.5);
            \coordinate (g2) at (120:2.5);
            
            \draw[->, line width=1.5pt, black!70] (O) -- (g1) node[right] {$g_1$};
            \draw[->, line width=1.5pt, black!70] (O) -- (g2) node[above left] {$g_2$};
            
            \coordinate (a1g1) at (0:1.5);
            \coordinate (a2g2) at (120:1.5);
            
            \draw[->, line width=1.5pt, cyan!80!blue] (O) -- (a1g1) node[below] {$\alpha_1 g_1$};
            \draw[->, line width=1.5pt, cyan!80!blue] (O) -- (a2g2) node[left] {$\alpha_2 g_2$};
            
            \coordinate (gIMTL) at ($(a1g1)+(a2g2)$);
            \draw[->, line width=1.8pt, blue!60!black] (O) -- (gIMTL) node[above right] {$g_{\text{IMTL-G}}$};
            \draw[dashed, gray] (a1g1) -- (gIMTL) -- (a2g2);
            \node[font=\large\bfseries] at (1.25, -1.2) {(b) IMTL-G};
        \end{scope}

        \begin{scope}[xshift=11cm]
            \coordinate (O) at (0,0);
            \coordinate (g1) at (0:2.5);
            \coordinate (g2) at (120:2.5);
            
            \draw[->, line width=1.5pt, black!70] (O) -- (g1) node[right] {$g_1$};
            \draw[->, line width=1.5pt, black!70] (O) -- (g2) node[above left] {$g_2$};
            
            \coordinate (Og2) at (90:2.165); 
            \coordinate (Og1) at (30:2.165);
            
            \draw[dashed, gray, opacity=0.6] (g2) -- (Og2);
            \draw[dashed, gray, opacity=0.6] (g1) -- (Og1);
            \draw[dashed, gray, opacity=0.6] (O) -- (Og2);
            \draw[dashed, gray, opacity=0.6] (O) -- (Og1);
            
            \coordinate (U1) at (30:1.5); 
            \coordinate (U2) at (90:1.5);
            
            \draw[->, line width=1.5pt, magenta] (O) -- (U1) node[right] {$\mathcal{U}(\mathcal{O}(g_1, g_2))$};
            \draw[->, line width=1.5pt, magenta] (O) -- (U2) node[left] {$\mathcal{U}(\mathcal{O}(g_2, g_1))$};
            
            \coordinate (gConFIG) at ($(U1)+(U2)$);
            \draw[->, line width=1.8pt, magenta!80!black] (O) -- (gConFIG) node[above] {$g_{\text{ConFIG}}$};
            \draw[dashed, gray] (U1) -- (gConFIG) -- (U2);
            
            \node[font=\large\bfseries] at (1.25, -1.2) {(c) ConFIG};
        \end{scope}

        \begin{scope}[xshift=16.5cm]
            \coordinate (O) at (0,0);
            \coordinate (g1) at (0:2.5);
            \coordinate (g2) at (120:2.5);
            
            \fill[green!10] (O) -- (30:3) arc (30:90:3) -- cycle;
            \draw[dashed, green!60!black, line width=1pt] (O) -- (30:3) node[right] {$\perp g_2$};
            \draw[dashed, green!60!black, line width=1pt] (O) -- (90:3) node[above] {$\perp g_1$};
            \node[green!60!black, font=\bfseries] at (60:2.8) {$\mathbb{H}$};
            
            \draw[->, line width=1.5pt, black!70] (O) -- (g1) node[right] {$g_1$};
            \draw[->, line width=1.5pt, black!70] (O) -- (g2) node[above left] {$g_2$};
            
            \coordinate (dstar) at (60:2.5);
            \draw[->, dashed, line width=1.2pt, blue] (O) -- (dstar) node[above right] {$d^*$};
            
            \coordinate (g1rot) at (30:2.5); 
            \coordinate (g2rot) at (90:2.5); 
            
            \draw[->, line width=1.5pt, red] (O) -- (g1rot) node[right] {$g_1^{\text{rot}}$};
            \draw[->, line width=1.5pt, red] (O) -- (g2rot) node[left] {$g_2^{\text{rot}}$};
            
            \draw[->, orange, thick] (g1) arc (0:30:2.5);
            \draw[->, orange, thick] (g2) arc (120:90:2.5);
            
            \coordinate (gHR) at ($(g1rot)+(g2rot)$);
            \draw[->, line width=1.8pt, red!70!black] (O) -- (gHR) node[above] {$g_{\text{HRGrad}}$};
            \draw[dashed, gray] (g1rot) -- (gHR) -- (g2rot);
            
            \node[font=\large\bfseries] at (1.25, -1.2) {(d) HRGrad (Ours)};
        \end{scope}
    \end{tikzpicture}
    }
    \caption{Geometric sketch comparison of gradient manipulation methods under conflicting tasks ($g_1, g_2$). (a) PCGrad relies on energy-clipping orthogonal projections. (b) IMTL-G rescales gradients to be equal-length. (c) ConFIG sums unit vectors of orthogonal components. (d) HRGrad isometrically rotates gradients into a Harmonized Cone $\mathbb{H}$, resolving conflicts while fully preserving kinetic energy.}
    \label{fig:method_comparison}
\end{figure*}
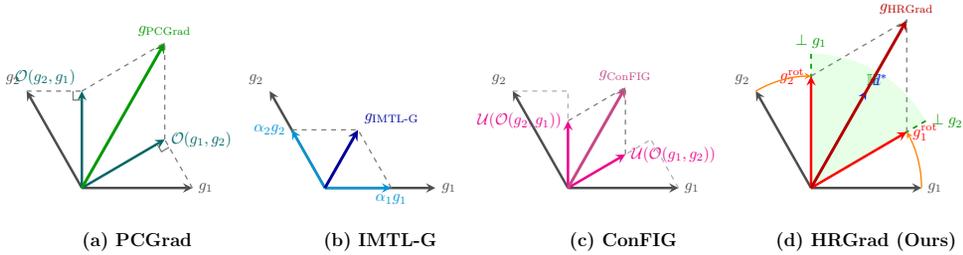

To address these fundamental theoretical flaws, we propose the \textbf{Harmonized Rotational Gradient (HRGrad)}. As visualized in Figure~\ref{fig:method_comparison}, HRGrad elevates conflict resolution from Euclidean projection to \textit{manifold isometric rotation}. Specifically, PCGrad (Figure~\ref{fig:method_comparison}(a)) resolves conflicts by summing orthogonal projections ($\mathcal{O}$), which inherently acts as an \textit{energy clipping} operator ($\|\mathcal{O}(g_i,g_j)\| \ll \|g_i\|$), irreversibly dissipating microscopic physical features. IMTL-G (Figure~\ref{fig:method_comparison}(b)) rescales gradients to be equal-length but may not fully guarantee manifold feasibility. ConFIG (Figure~\ref{fig:method_comparison}(c)) attempts to mitigate magnitude loss by summing unit vectors $\mathcal{U}$ of the orthogonal components, but the underlying direction still relies on a lossy projection basis. In contrast, HRGrad (Figure~\ref{fig:method_comparison}(d)) explicitly constructs the \textit{Harmonized Cone} $\mathbb{H} = \mathbb{K} \cap \mathbb{K}^*$ and extracts a physically valid anchor $d^*$. Conflicting gradients are deflected into $\mathbb{H}$ via \textit{manifold isometric rotation} (orange arcs), mathematically ensuring absolute non-conflict while flawlessly preserving $100\%$ of the kinetic energy ($\|g_i^{\text{rot}}\| \equiv \|g_i\|$). This approach is strictly bounded within a physically valid harmonized cone, which operates in four rigorous stages.

\paragraph{\textbf{Stage 1: Harmonized Cone and Physical Anchoring}}
A physically valid multi-scale update direction must be both feasible (expressible as a non-negative linear combination of actual physical mechanisms) and non-conflicting across all $\varepsilon$ scales. We define the Primal Gradient Cone $\mathbb{K} = \{ G\lambda \mid \lambda \in \mathbb{R}_+^m \}$ to guarantee physical feasibility, and the Dual Gradient Cone $\mathbb{K}^* = \{ y \in \mathbb{R}^D \mid G^\top y \ge \mathbf{0}_m \}$ to guarantee absolute non-conflict.

HRGrad restricts the optimization search space to their exact geometric intersection, the \textit{Harmonized Cone} $\mathbb{H} := \mathbb{K} \cap \mathbb{K}^*$. By leveraging the Double Description Method (DDM) to solve the algebraic constraint $G^\top G \lambda \ge \mathbf{0}_m$, we analytically extract the extreme rays $\Pi = [\pi_1, \cdots, \pi_p]$ of the feasible weight space. The physical extreme rays in the gradient space are given by $r_j = G\pi_j$. We then determine the absolute physical target direction $d^* \in \text{Int}(\mathbb{H})$ as the normalized centroid of these rays:
\begin{equation}
    d^* = \mathcal{U}\left( \sum_{j=1}^p \mathcal{U}(r_j) \right),
\end{equation}
where $\mathcal{U}(x) = x/(\|x\| + \delta)$ is the normalization operator, and $\delta>0$ is a small constant introduced for numerical stability. This strictly locks the neural network's evolution trajectory onto the AP manifold.

\paragraph{\textbf{Stage 2: Manifold Isometric Rotation Operator}}
For any normalized task gradient $\bar{g}_i = \mathcal{U}(g_i)$ that conflicts with the physical anchor (i.e., $\bar{g}_i^\top d^* < 0$), HRGrad completely abandons the projection operator. Instead, we construct a parameterized rotation operator on the 2D Riemannian tangent plane $\mathcal{T}_i = \text{span}(\bar{g}_i, d^*)$. Using the Gram-Schmidt process, we first extract the feasible orthogonal reference state:
\begin{equation}
    w_i = \mathcal{U}\left( d^* - (\bar{g}_i^\top d^*) \bar{g}_i \right).
\end{equation}
The parameterized $SO(2)$ isometric rotation operator $r_i(\alpha_i)$ with a deflection angle $\alpha_i \in [0, \pi/2]$ is formulated as:
\begin{equation}
    r_i(\alpha_i) = \cos(\alpha_i)\bar{g}_i + \sin(\alpha_i)w_i.
\end{equation}
This operator guarantees \textit{Isometric Fidelity} ($\|r_i(\alpha_i)\| \equiv 1$). The directional deflection is completely decoupled from the vector norm, ensuring that $100\%$ of the micro-scale kinetic energy is preserved during conflict resolution.

\paragraph{\textbf{Stage 3: Optimizing MER via MCAE}}
Blindly rotating the gradient to a completely orthogonal boundary constitutes a local hard-adjustment, which severely distorts task-specific physical features. To achieve a global Multi-task Equilibrium Relationship (MER), we seek the optimal rotation angles $\boldsymbol{\alpha}^* = [\alpha_1^*, \dots, \alpha_m^*]$ by minimizing a global energy functional:
\begin{equation}
    \min_{\boldsymbol{\alpha} \in [0, \pi/2]^m} \mathcal{L}(\boldsymbol{\alpha}) = \underbrace{ \frac{1}{m(m-1)} \sum_{i<j} \left[ 1 - r_i(\alpha_i)^\top r_j(\alpha_j) \right] }_{\text{Global Alignment (Suppress UCAE)}} + \lambda \underbrace{ \frac{1}{4m} \sum_{i=1}^m \| r_i(\alpha_i) - \bar{g}_i \|^2 }_{\text{Global Proximity (Suppress OCAE)}}. \label{eq:mer_objective}
\end{equation}
A second-order Taylor expansion of the Miss-Correction Angular Error (MCAE) demonstrates that any deviation from the optimal angle $\alpha_i^*$ introduces quadratic optimization penalties. Hard projections induce severe Over-Correction Angular Error (OCAE), obliterating microscopic specificities, while insufficient rotation leads to Under-Correction Angular Error (UCAE). 

Furthermore, to address the severe kinetic stiffness that varies drastically across $\varepsilon$, we introduce an adaptive search strategy for the inner optimization of Eq.~\ref{eq:mer_objective}. Let $\Delta_i^{(t)} = (L_i^{(t)} - L_i^{(t-1)}) / (L_i^{(t-1)} + \delta)$ be the relative loss change, and $s_t = \text{std}(\Delta_1^{(t)}, \dots, \Delta_m^{(t)})$ be its standard deviation. The number of inner SGD steps $\alpha_{\text{steps}}^{(t)}$ is dynamically adjusted:
\begin{equation}
    \alpha_{\text{steps}}^{(t)} = \left\lfloor \alpha_{\text{min}} + (\alpha_{\text{max}} - \alpha_{\text{min}}) \frac{s_t}{s_t + k_{\text{std}}} \right\rfloor.
\end{equation}
This mechanism allocates more computational effort to accurately locate $\boldsymbol{\alpha}^*$ when the multi-scale system enters highly stiff transitional regimes (where $s_t$ surges), strictly preventing OCAE distortion on high-frequency micro-features.

\paragraph{\textbf{Stage 4: Magnitude Restoration and Fair Aggregation}}
After achieving the perfect angular harmony, HRGrad flawlessly restores the intrinsic physical magnitudes to prevent macroscopic regimes from dominating the update. The physically restored gradient is given by:
\begin{equation}
    g_i^{\text{rot}} = \begin{cases} 
      \|g_i\| \cdot r_i(\alpha_i^*), & \text{if } \bar{g}_i^\top d^* < 0 \\
      g_i, & \text{otherwise}
   \end{cases}.
\end{equation}
To ensure scale-invariance and fair learning across the entire continuous spectrum, we adopt a pseudo-inverse equal-length aggregation. We construct the normalized orientation matrix $M = [ \mathcal{U}(g_1^{\text{rot}}), \mathcal{U}(g_2^{\text{rot}}), \cdots, \mathcal{U}(g_m^{\text{rot}}) ]$. Utilizing the Moore-Penrose pseudoinverse $M^\dagger$, we derive the globally fair consensus direction $g_u$ that maintains equal projection lengths across all rotated tasks:
\begin{equation}
    \label{eq:HRGrad_gu}
    g_u = \mathcal{U}\left[ (M^\dagger)^\top \mathbf{1}_m \right], 
\end{equation}
where $\mathbf{1}_m$ is a vector of ones. The final HRGrad update gradient is formulated by projecting the physically restored gradients onto this fair direction:
\begin{equation}
    \label{eq:HRGrad_final}
    g_{\text{HRGrad}} = \mathcal{G}_{\text{HRGrad}}(g_1, \cdots, g_m) := \left( \sum_{i=1}^m (g_i^{\text{rot}})^\top g_u \right) g_u.
\end{equation}

\subsection{Model architecture}

To seamlessly integrate HRGrad into deep learning optimizers, we apply the $\mathcal{G}_{\text{HRGrad}}$ operator directly to the pseudo-first-moments of the gradients within the Adam optimizer. This elegantly stabilizes the multiscale optimization trajectory. The complete process is detailed in Algorithm \ref{alg:m_HRGrad}.

\begin{algorithm}[ht!]
    \small
    \caption{HRGrad for Multi-task Learning}\label{alg:m_HRGrad}
    \begin{algorithmic}[1]
        \REQUIRE Initial network weights $\theta_0$, learning rate $\gamma$, Adam coefficients $\beta_1, \beta_2, \epsilon$.
        \REQUIRE HRGrad hyper-parameters $\alpha_{\min}, \alpha_{\max}, k_{\text{std}}, \lambda, \delta$.
        \STATE Initialize first momentum $[m_{g_1,0}, m_{g_2,0}, \dots, m_{g_m,0}] \gets [0,0,\dots,0]$ and second momentum $v_{0} \gets 0$.
        \STATE Initialize step counters $[t_{g_1}, t_{g_2}, \dots, t_{g_m}] \gets [0,0,\dots,0]$.
        \FOR{$t = 1, 2, \dots$}
        \STATE $i \gets (t-1)\bmod m+1$
        \STATE $t_{g_i} \gets t_{g_i}+1$
        \STATE $m_{g_i,t_{g_i}} \gets \beta_1 m_{g_i,t_{g_i}-1}+(1-\beta_1) \nabla_{\theta_{t-1}}\mathcal{L}_{i}$
        \STATE Compute bias-corrected momentums: $[\hat{m}_{g_1}, \dots, \hat{m}_{g_m}] \gets \left[ \frac{m_{g_1,t_{g_1}}}{1-\beta_1^{t_{g_1}}}, \dots, \frac{m_{g_m,t_{g_m}}}{1-\beta_1^{t_{g_m}}} \right]$

        \STATE \textbf{HRGrad Operator $\mathcal{G}_{\text{HRGrad}}$ applied to $\hat{G} = [\hat{m}_{g_1}, \dots, \hat{m}_{g_m}]$:}

        \STATE \emph{Stage 1 (Harmonized cone and physical anchor).}
        \STATE Solve the DDM constraint $\hat{G}^\top \hat{G}\,\lambda \geq \mathbf{0}_m$ to extract extreme rays $\Pi = [\pi_1, \dots, \pi_p]$ of $\mathbb{H} = \mathbb{K}\cap\mathbb{K}^*$.
        \STATE Compute physical extreme rays $r_k \gets \hat{G}\pi_k$ for $k=1,\dots,p$ and the physical anchor
        $d^* \gets \mathcal{U}\!\left( \sum_{k=1}^p \mathcal{U}(r_k) \right) \in \text{Int}(\mathbb{H})$.

        \STATE \emph{Stage 2 (Per-task rotation references).}
        \STATE Form the unit gradients $\bar{m}_{g_j} \gets \mathcal{U}(\hat{m}_{g_j})$ for $j=1,\dots,m$ and identify the conflicting set $\mathcal{C} \gets \{\,j : \bar{m}_{g_j}^\top d^* < 0\,\}$.
        \FOR{$j \in \mathcal{C}$}
            \STATE $w_j \gets \mathcal{U}\!\left( d^* - (\bar{m}_{g_j}^\top d^*)\bar{m}_{g_j} \right)$.
        \ENDFOR

        \STATE \emph{Stage 3 (Joint MER optimization with adaptive steps).}
        \STATE Compute relative loss variance $s_t$ and adaptive step count $\alpha_{\text{steps}}^{(t)}$.
        \STATE Fix $\alpha_j \equiv 0$ for $j \notin \mathcal{C}$ and run $\alpha_{\text{steps}}^{(t)}$ inner SGD iterations on the MER objective $\mathcal{L}(\boldsymbol{\alpha})$ in Eq.~\ref{eq:mer_objective}, jointly over $\{\alpha_j\}_{j \in \mathcal{C}} \subset [0,\pi/2]$, to obtain $\boldsymbol{\alpha}^* = [\alpha_1^*,\dots,\alpha_m^*]$.

        \STATE \emph{Stage 4 (Magnitude restoration and fair aggregation).}
        \FOR{$j = 1$ to $m$}
            \IF{$j \in \mathcal{C}$}
                \STATE $r_j \gets \cos(\alpha_j^*)\bar{m}_{g_j} + \sin(\alpha_j^*)w_j$,\quad
                $m_j^{\text{rot}} \gets \|\hat{m}_{g_j}\|\, r_j$.
            \ELSE
                \STATE $m_j^{\text{rot}} \gets \hat{m}_{g_j}$.
            \ENDIF
        \ENDFOR
        \STATE Form the normalized rotated matrix $M \gets [\mathcal{U}(m_1^{\text{rot}}), \dots, \mathcal{U}(m_m^{\text{rot}})]$.
        \STATE Compute the fair aggregation direction $g_u \gets \mathcal{U}\!\left[ (M^\dagger)^\top \mathbf{1}_m \right]$.
        \STATE $\hat{m}_g \gets \mathcal{G}_{\text{HRGrad}}(\hat{m}_{g_1},\dots,\hat{m}_{g_m}) := \left( \sum_{j=1}^m (m_j^{\text{rot}})^\top g_u \right) g_u$.

        \STATE Estimate the unified gradient: $g_{c} \gets [\hat{m}_g(1-\beta_1^t)-\beta_1 m_{t-1}]/(1-\beta_1)$.
        \STATE $m_{t} \gets \beta_1 m_{t-1}+(1-\beta_1)g_{c}$.
        \STATE $v_{t} \gets \beta_2 v_{t-1}+(1-\beta_2)g_{c}^2$.
        \STATE $\hat{v}_g \gets v_{t}/(1-\beta_2^t)$.
        \STATE Update network weights: $\theta_t \gets \theta_{t-1} - \gamma\,\hat{m}_g/(\sqrt{\hat{v}_{g}}+\epsilon)$.
        \ENDFOR
    \end{algorithmic}
\end{algorithm}


\subsection{Pre-training}\label{sec:train}

\subsection{Fast fine-tuning}\label{tune}

\subsection{Convergence analysis}\label{sec:convergence}

We establish the convergence guarantees of HRGrad for both convex and non-convex settings.
We first derive the key geometric properties arising from the four-stage construction, and then exploit them to prove descent lemmas and convergence theorems.

Recall from Stage~2 and Stage~4 that the output of the rotation and magnitude-restoration process yields
\begin{equation}\label{eq:grot_def}
  g_i^{\mathrm{rot}} = \begin{cases}
    \|g_i\|\,r_i(\alpha_i^*), & \text{if } \bar{g}_i^\top d^* < 0, \\
    g_i, & \text{otherwise,}
  \end{cases}
\end{equation}
where $r_i(\alpha_i^*) = \cos(\alpha_i^*)\bar{g}_i + \sin(\alpha_i^*)w_i$ is the $SO(2)$ isometric rotation with $\alpha_i^* \in [0,\pi/2]$,
$\bar{g}_i = \mathcal{U}(g_i)$, and $w_i = \mathcal{U}\bigl(d^* - (\bar{g}_i^\top d^*)\bar{g}_i\bigr)$.

\begin{lemma}[Geometric Properties of HRGrad]\label{lem:hrgrad_properties}
Let $\{g_i^{\mathrm{rot}}\}_{i=1}^m$ and $g_{\mathrm{HRGrad}}$ be defined in Eqs.~\eqref{eq:grot_def},~\eqref{eq:HRGrad_gu}--\eqref{eq:HRGrad_final},
and let $M = [\mathcal{U}(g_1^{\mathrm{rot}}), \ldots, \mathcal{U}(g_m^{\mathrm{rot}})] \in \mathbb{R}^{D \times m}$ have full column rank.
Then:
\begin{itemize}
  \item[\textit{(i)}] \textbf{(Isometric Fidelity)} $\|g_i^{\mathrm{rot}}\| = \|g_i\|$ for all $i = 1, \ldots, m$.
  \item[\textit{(ii)}] \textbf{(Non-Conflict)} $(g_i^{\mathrm{rot}})^\top g_{\mathrm{HRGrad}} \geq 0$ for all $i$,
  with strict inequality whenever $\|g_{\mathrm{HRGrad}}\| > 0$.
  \item[\textit{(iii)}] \textbf{(Equal Cosine Similarity)} There exists a uniform constant $S_c > 0$ such that
  \begin{equation}\label{eq:equal_sc}
    \frac{(g_i^{\mathrm{rot}})^\top g_{\mathrm{HRGrad}}}{\|g_i^{\mathrm{rot}}\|\,\|g_{\mathrm{HRGrad}}\|} = S_c, \quad \forall\; i = 1, \ldots, m.
  \end{equation}
  \item[\textit{(iv)}] \textbf{(Aggregate Product Identity)}
  $\displaystyle\Bigl(\sum_{i=1}^m g_i^{\mathrm{rot}}\Bigr)^\top g_{\mathrm{HRGrad}} = \|g_{\mathrm{HRGrad}}\|^2$.
\end{itemize}
\end{lemma}

\begin{proof}
\textit{(i)}\quad For conflicting tasks: $r_i(\alpha_i^*)$ satisfies $\bar{g}_i \perp w_i$ and $\|\bar{g}_i\| = \|w_i\| = 1$,
so $\|r_i(\alpha_i^*)\| = \sqrt{\cos^2(\alpha_i^*) + \sin^2(\alpha_i^*)} = 1$.
By Stage~4 magnitude restoration $g_i^{\mathrm{rot}} = \|g_i\|\,r_i(\alpha_i^*)$, hence $\|g_i^{\mathrm{rot}}\| = \|g_i\|$.
For non-conflicting tasks $g_i^{\mathrm{rot}} = g_i$ by Eq.~\eqref{eq:grot_def}.

\textit{(ii)}\quad By Stages~1--3, every $g_i^{\mathrm{rot}}$ is deflected into the Harmonized Cone
$\mathbb{H} = \mathbb{K} \cap \mathbb{K}^*$, and the anchor $d^*$ is the normalized centroid
of the extreme rays of $\mathbb{H}$, so the fair aggregation direction $g_u \in \mathrm{Int}(\mathbb{H}) \subseteq \mathbb{K}^*$.
The dual cone property of $\mathbb{K}^*$ gives $(g_i^{\mathrm{rot}})^\top g_u \geq 0$ for all $i$.
From Eq.~\eqref{eq:HRGrad_final}:
\begin{equation}
  (g_i^{\mathrm{rot}})^\top g_{\mathrm{HRGrad}}
  = \underbrace{\Bigl(\sum_{j=1}^m (g_j^{\mathrm{rot}})^\top g_u\Bigr)}_{\geq\,0}
    \cdot \underbrace{(g_i^{\mathrm{rot}})^\top g_u}_{\geq\,0} \geq 0.
\end{equation}
Equality holds iff $(g_i^{\mathrm{rot}})^\top g_u = 0$ for all $i$, i.e., $\|g_{\mathrm{HRGrad}}\| = 0$.

\textit{(iii)}\quad Since $M$ has full column rank, $(M^\dagger)^\top = M(M^\top M)^{-1}$.
Hence $g_u \propto M(M^\top M)^{-1}\mathbf{1}_m$, and
\begin{equation}
  M^\top g_u
  = \frac{M^\top M(M^\top M)^{-1}\mathbf{1}_m}{\bigl\|M(M^\top M)^{-1}\mathbf{1}_m\bigr\|}
  = \frac{\mathbf{1}_m}{\bigl\|M(M^\top M)^{-1}\mathbf{1}_m\bigr\|}
  =: S_c\,\mathbf{1}_m,
\end{equation}
so $\mathcal{U}(g_i^{\mathrm{rot}})^\top g_u = S_c > 0$ uniformly in $i$.
Because $g_{\mathrm{HRGrad}} \propto g_u$, the cosine similarity in Eq.~\eqref{eq:equal_sc} is constant for all $i$.

\textit{(iv)}\quad Summing over all $i$ and applying \textit{(iii)}:
\begin{equation}
  \Bigl(\sum_{i=1}^m g_i^{\mathrm{rot}}\Bigr)^\top g_{\mathrm{HRGrad}}
  = \sum_{i=1}^m \|g_i^{\mathrm{rot}}\|\,\|g_{\mathrm{HRGrad}}\|\,S_c
  = \|g_{\mathrm{HRGrad}}\|^2,
\end{equation}
where $\|g_{\mathrm{HRGrad}}\| = \bigl(\sum_i (g_i^{\mathrm{rot}})^\top g_u\bigr)\|g_u\| = \sum_i \|g_i^{\mathrm{rot}}\|\,S_c$
follows from $\|g_u\| = 1$ and each $(g_i^{\mathrm{rot}})^\top g_u = \|g_i^{\mathrm{rot}}\|\,S_c$.
\end{proof}

\begin{remark}
Isometric Fidelity ensures $\|g_{\mathrm{HRGrad}}\| = 0$ if and only if $\|g_i\| = 0$ for all $i$,
since $\|g_{\mathrm{HRGrad}}\| = S_c\sum_i \|g_i\|$ with $S_c > 0$.
This prevents the spurious premature stagnation caused by gradient cancellation in projection-based methods
(where $\|\mathcal{P}(g_i)\| \to 0$ even when $\|g_i\| \gg 0$), guaranteeing that the optimizer
halts only at genuine joint stationary points of all task losses.
\end{remark}

To formalize the descent analysis, we introduce the proxy aggregate objective
$\hat{\mathcal{L}}(\theta) := \frac{1}{m}\sum_{i=1}^m \hat{\mathcal{L}}_i(\theta)$,
where each $\hat{\mathcal{L}}_i$ is the locally linear surrogate with
$\nabla_\theta \hat{\mathcal{L}}_i(\theta) = g_i^{\mathrm{rot}}(\theta)$.
By Lemma~\ref{lem:hrgrad_properties}(i), $\hat{\mathcal{L}}_i$ and $\mathcal{L}_i$ share
the same critical points: $\nabla_\theta \hat{\mathcal{L}}_i = 0 \Leftrightarrow \nabla_\theta \mathcal{L}_i = 0$.

\begin{theorem}[Convergence in the Convex Setting]\label{thm:hrgrad_convex}
Assume (a) the proxy objectives $\{\hat{\mathcal{L}}_i\}_{i=1}^m$ are convex and differentiable;
(b) the aggregate proxy gradient $\hat{g} := \sum_{i=1}^m g_i^{\mathrm{rot}}$ is $L$-Lipschitz continuous
with constant $L > 0$.
Then updating along $g_{\mathrm{HRGrad}}$ with step size $\gamma \leq \frac{2}{L}$ guarantees
\begin{equation}\label{eq:convex_descent}
  \hat{\mathcal{L}}(\theta^+) \leq \hat{\mathcal{L}}(\theta)
  - \gamma\!\left(1 - \frac{L\gamma}{2}\right)\!\|g_{\mathrm{HRGrad}}\|^2
  \leq \hat{\mathcal{L}}(\theta),
\end{equation}
converging to either a location where $\|g_{\mathrm{HRGrad}}\| = 0$ or the global optimum.
\end{theorem}

\begin{proof}
The $L$-Lipschitz gradient condition yields the quadratic upper bound:
\begin{equation}
  \hat{\mathcal{L}}(\theta^+)
  \leq \hat{\mathcal{L}}(\theta) + \hat{g}^\top(\theta^+ - \theta) + \frac{L}{2}\|\theta^+ - \theta\|^2.
\end{equation}
Substituting $\theta^+ = \theta - \gamma g_{\mathrm{HRGrad}}$:
\begin{equation}\label{eq:hrgrad_quad}
  \hat{\mathcal{L}}(\theta^+)
  \leq \hat{\mathcal{L}}(\theta) - \gamma\,\hat{g}^\top g_{\mathrm{HRGrad}}
  + \frac{L\gamma^2}{2}\|g_{\mathrm{HRGrad}}\|^2.
\end{equation}
Applying the Aggregate Product Identity (Lemma~\ref{lem:hrgrad_properties}(iv)),
$\hat{g}^\top g_{\mathrm{HRGrad}} = \|g_{\mathrm{HRGrad}}\|^2$.
Inserting into Eq.~\eqref{eq:hrgrad_quad}:
\begin{equation}
  \hat{\mathcal{L}}(\theta^+)
  \leq \hat{\mathcal{L}}(\theta)
  - \gamma\!\left(1 - \frac{L\gamma}{2}\right)\!\|g_{\mathrm{HRGrad}}\|^2.
\end{equation}
When $\gamma \leq \frac{2}{L}$ the coefficient $(1 - \frac{L\gamma}{2}) \geq 0$,
so Eq.~\eqref{eq:convex_descent} holds.
Equality requires $\|g_{\mathrm{HRGrad}}\| = 0$; by convexity this is the global optimality condition.
\end{proof}

\begin{theorem}[Convergence in the Non-Convex Setting]\label{thm:hrgrad_nonconvex}
Assume (a) the proxy objectives $\{\hat{\mathcal{L}}_i\}_{i=1}^m$ are differentiable but possibly non-convex;
(b) $\hat{g} = \sum_{i=1}^m g_i^{\mathrm{rot}}$ is $L$-Lipschitz continuous.
Then updating along $g_{\mathrm{HRGrad}}$ with step size $\gamma \leq \frac{1}{L}$ satisfies the ergodic bound:
\begin{equation}\label{eq:nonconvex_rate}
  \min_{1 \leq k \leq K}\left\|\sum_{i=1}^m g_i^k\right\|^2
  \leq \frac{2\bigl[\hat{\mathcal{L}}(\theta^0) - \hat{\mathcal{L}}^*\bigr]}{\gamma\,\alpha^2\,K},
\end{equation}
where $g_i^k = \nabla_\theta \mathcal{L}_i(\theta^k)$,\;
$\hat{\mathcal{L}}^* = \inf_\theta \hat{\mathcal{L}}(\theta)$, and
$\alpha = \min_{1 \leq k \leq K} S_c^k > 0$ is the minimum equal cosine similarity across all iterations.
Consequently, $\min_{k} \bigl\|\sum_i g_i^k\bigr\| \to 0$ as $K \to \infty$.
\end{theorem}

\begin{proof}
Applying the descent lemma with $\gamma \leq \frac{1}{L}$ and the Aggregate Product Identity
(Lemma~\ref{lem:hrgrad_properties}(iv)):
\begin{equation}
  \hat{\mathcal{L}}(\theta^{k+1})
  \leq \hat{\mathcal{L}}(\theta^k) - \frac{\gamma}{2}\|g_{\mathrm{HRGrad}}^k\|^2.
\end{equation}
By Isometric Fidelity (Lemma~\ref{lem:hrgrad_properties}(i)) and Equal Cosine Similarity (iii):
\begin{equation}
  \|g_{\mathrm{HRGrad}}^k\|
  = \sum_{i=1}^m \|g_i^{k,\mathrm{rot}}\|\,S_c^k
  = \sum_{i=1}^m \|g_i^k\|\,S_c^k
  \geq \left\|\sum_{i=1}^m g_i^k\right\| S_c^k,
\end{equation}
where the final step uses the triangle inequality $\sum_i \|g_i\| \geq \|\sum_i g_i\|$. Therefore:
\begin{equation}
  \hat{\mathcal{L}}(\theta^{k+1})
  \leq \hat{\mathcal{L}}(\theta^k)
  - \frac{\gamma}{2}(S_c^k)^2\left\|\sum_{i=1}^m g_i^k\right\|^2.
\end{equation}
Telescoping from $k = 1$ to $K$ and using $\hat{\mathcal{L}} \geq \hat{\mathcal{L}}^*$:
\begin{equation}
  \sum_{k=1}^K (S_c^k)^2 \left\|\sum_{i=1}^m g_i^k\right\|^2
  \leq \frac{2}{\gamma}\bigl[\hat{\mathcal{L}}(\theta^0) - \hat{\mathcal{L}}^*\bigr].
\end{equation}
Setting $\alpha = \min_k S_c^k > 0$ and dividing by $\alpha^2 K$ gives Eq.~\eqref{eq:nonconvex_rate}.
As $K \to \infty$ the right-hand side vanishes,
hence $\min_k \bigl\|\sum_i g_i^k\bigr\| \to 0$, establishing convergence to a stationary point.
\end{proof}

\section{Numerical results}\label{sec:result}
In this section, we present the numerical results of the proposed HRGrad for several multiscale kinetic equations, including the Boltzmann--BGK equation, the linear transport equation, and two additional kinetic models with more challenging asymptotic structures.
Besides the classical BGK and linear transport benchmarks, we further consider the ES-BGK equation and the linear semiconductor Boltzmann--Poisson equation.
The former extends the standard BGK model by replacing the isotropic Maxwellian equilibrium with an ellipsoidal Gaussian depending on the stress tensor, thereby allowing the correct Prandtl number in the fluid limit.
The latter extends the linear transport setting by incorporating a self-consistent electric field, a velocity-space drift term, and a Poisson coupling, with the drift--diffusion--Poisson system as its asymptotic limit.
These two additional examples are chosen to examine whether HRGrad remains effective when the multiscale conflict is caused not only by stiff relaxation, but also by high-order moment closure, anisotropy, self-consistent fields, and boundary layers.
Extensive numerical results are presented for several problems chosen from kinetic regimes $(\varepsilon \approx O(1))$ to hydrodynamic (diffusive) regimes $(\varepsilon \rightarrow 0)$, in order to compare our method and existing MTL methods. 
To verify and compare the performance of existing MTL methods and our method, we present both $1$D and $2$D numerical results for several problems chosen from rarefied regimes ($\varepsilon \approx O(1)$) to hydrodynamic (diffusive) regimes ($\varepsilon \to 0$).
To validate the effectiveness of our method, we conduct all numerical experiments with the initial and boundary conditions kept consistent with those of APNNs \cite{jin2023ap,jin2024ap}. 
Furthermore, we consider additional solution tasks for Knudsen numbers $\varepsilon\in(0,1]$.
All experiments are conducted on an Nvidia RTX-3090-24GB GPU.
Code and data for the following experiments are available at \url{https://github.com/liangzhangyong/HRGrad}.

\subsection{MTL baselines}
We compared the performance of the proposed HRGrad method with ten popular MTL baselines. 
Besides PCGrad and IMTL-G  from before, we also compare to Linear scalarization baseline (LS) \cite{zhang2018overview}, Uncertainty Weighting (UW) \cite{kendall2018uw}, Dynamic Weight Average (DWA) \cite{liu2019dwa}, Gradient Sign Dropout (GradDrop) \cite{chen2020drop}, Conflict-Averse Gradient Descent (CAGrad) \cite{liu2021conflict}, Random Loss Weighting (RLW) \cite{lin2021rlw}, Nash bargaining solution for MTL (Nash-MTL) \cite{navon2022nash}, and Fast Adaptive Multitask Optimization (FAMO) \cite{liu2023famo}. 
We use two metrics to measure the performance of different methods: the mean rank metric ($MR$, lower ranks being better) \cite{navon2022nash,liu2023famo}, which represents the average rank across all tasks. 
An $MR$ of 1 means consistently outperforming all others.
In addition, we consider the average F1 score ($\overline{F_1}$, larger is better) to measure the average performance on all tasks.

\paragraph{MGDA}
In MGDA \cite{sener2018multi}, the gradient vectors are first normalized, and the update step is taken in the direction of the minimum-norm point within the convex hull formed by their normalized convex combination. 
We employed the $\ell_2$ normalization, proposed in MGDA, which has the advantage of providing a stable definition of the update direction.

\paragraph{CAGrad}
In CAGrad \cite{liu2021conflict}, the hyperparameter $c \in [0,1)$ constrains the update direction to remain within a certain distance from the average of the loss gradients, and we set $c = 0.5$ in our experiments.

\paragraph{ConFIG}
In ConFIG \cite{liu2024config}, the update direction is originally obtained by solving a least-squares problem to find a vector whose inner product with each loss gradient equals 1.
However, this procedure produced NaN values in all five independent runs on NS2d-C, Volterra1d, and Volterra system.
Therefore, in our implementation, instead of using the unstable least-squares method, we computed the pseudo-inverse of the loss gradient matrix and multiplied it by $\mathbf{1}_m$ to obtain the update direction instead.


\subsection{Experiment setting}
Some settings for the numerical experiments are stated as follows. 
Our method is a gradient preconditioner that can be embedded into existing MTL methods. 
The Adam optimizer of the gradient descent algorithm is commonly used for MTL optimization of the task loss functions. 
The initial value of the parameters in all numerical experiments is generated by Xavier initialization. 
All the hyperparameters are chosen for the best performance after trying these experiments.
We adopt the APNNs architecture and automatically satisfy the constraint settings in each task.
Meanwhile, we adopt the meta-auto-decoder (MAD) method to construct the MTL framework for faster convergence speed without losing accuracy compared to other MTL methods across tasks.

The activation function we used is $\sigma(x) = x / (1 + \exp(-x))$ for BGK problems and $\sigma(x) = \tanh(x)$ for linear transport problems.
In particular, the spatial domain of interest, denoted as $\mathcal{D}$, includes the interval $[0, 1]$ for the linear transport equation and $[-0.5, 0.5]$ for the Boltzmann-BGK equation. 
Additionally, the parameter $V$ is assigned a value of 10.
The number of quadrature points is $30$ for the linear transport equation and $64$ for the Boltzmann-BGK equation.
To train the networks, the Adam \cite{kingma2014adam} version of the gradient descent method is used to solve the optimization problem with Xavier initialization.
In practice, we need to tune the hyperparameters, such as neural network architecture, learning rate, and batch size, to obtain a good level of accuracy \cite{lu2021deepxde}. 
As a matter of experience, one can tune the weights of loss terms to equalize them, and a decreasing annealing schedule for the learning rate is used to achieve better numerical performance.
We use an exponential decay strategy for an initial learning rate $\eta_0 = 10^{-3}$ with a decay rate of $\gamma = 0.96$ and a decay step of $p = 200$ iterations
\begin{equation}
    \eta_i = \eta_0 \cdot \gamma^{\lfloor \frac{i}{p} \rfloor},
\end{equation}
where, the variable $i$ represents the current $i-$th iteration step, and the symbol $\lfloor \cdot \rfloor$ denotes the floor function.

The reference solutions are obtained by standard finite difference methods. 
For most of the time we will check the relative $\ell^2$ error of the density $\rho(x)$ between our method and reference solutions, e.g. for $1$d case,
\begin{equation}
    \text{error} := \sqrt{
        \frac{\sum_j |\rho_{\theta, j} - \rho_j|^2}
        {\sum_j |\rho_j|^2}
    }.
\end{equation}

\paragraph{\textbf{Multiscale sampling}}
To enable effective pretraining across the full spectrum of Knudsen numbers $\varepsilon \in [\varepsilon_{\min}, 1]$ with $\varepsilon_{\min} > 0$ sufficiently small (typically $\varepsilon_{\min} = 10^{-6}$ or $10^{-4}$), a carefully designed sampling strategy for the scale parameter $\varepsilon$ is crucial. 

Training difficulty is strongly heterogeneous across regimes. 
The small-$\varepsilon$ regime is significantly harder: the equations become stiff, gradients are larger and more sensitive, and the network must accurately capture rapid relaxation while preserving macroscopic evolution and resolving thin layers. 
Naive uniform sampling in the linear scale, i.e., 
\begin{equation}
\varepsilon \sim \mathcal{U}(\varepsilon_{\min}, 1),
\end{equation}
severely undersamples the difficult small-$\varepsilon$ region. 
Such a sampling strategy biases training toward the transport-dominated regime, as the small-$\varepsilon$ interval is effectively ignored in a linear scale. 
This imbalance severely impedes the convergence and triggers failure modes, particularly in the challenging continuum limit.

To mitigate this, we adopt a log-uniform (log-scale) sampling strategy, which naturally assigns higher sampling density to small $\varepsilon$. 
Specifically, we introduce an auxiliary variable $\xi$ uniformly distributed over the logarithmic range
\begin{equation}
\xi \sim \mathcal{U}\bigl(\log\varepsilon_{\min}, \, 0\bigr),
\end{equation}
The scale parameter is then obtained via exponential mapping:
\begin{equation}
\varepsilon = \exp(\xi),
\end{equation}
The resulting probability density function for $\varepsilon$ is
\begin{equation}
p(\varepsilon) = \frac{1}{\varepsilon \, (\log 1 - \log \varepsilon_{\min})} = \frac{1}{\varepsilon \, \log(1/\varepsilon_{\min})}, \qquad \varepsilon \in [\varepsilon_{\min}, 1].
\end{equation}
Thus, the sampling density is proportional to $1/\varepsilon$, ensuring exponentially higher sampling frequency in decades where stiffness and asymptotic complexity are greatest.

In practice, during each training iteration, we independently sample a batch of $N_b$ values $\{\varepsilon_k\}_{k=1}^{N_b}$ according to this log-uniform distribution:
\begin{equation}
\varepsilon_k = \exp\!\bigl(\xi_k\bigr), \qquad \xi_k \sim \mathcal{U}\bigl(\log \varepsilon_{\min}, \, 0\bigr), \quad k=1,\dots,N_b.
\end{equation}
For each sampled $\varepsilon_k$, we generate the corresponding task-specific collocation points in the domain as well as boundary and initial condition points, and compute the per-task residual losses using the APNNs formulation. 
The resulting task gradients $\{g_k\}$ are then aggregated and corrected via HRGrad to produce the final parameter update.

This log-uniform strategy aligns with the multiplicative nature of multiscale transitions in kinetic theory and significantly improves training stability, convergence speed, and accuracy in the small-$\varepsilon$ regime.
It reduces the occurrence of failure modes observed under linear uniform sampling and enhances extrapolation capability to unseen $\varepsilon \to 0^+$.

For further refinement, one may incorporate adaptive importance resampling: periodically estimate the difficulty of logarithmically spaced $\varepsilon$-bins via moving-average residual loss or gradient norm, and adjust the sampling distribution to up-weight currently hard regimes. 
However, the baseline log-uniform sampling already provides substantial robustness when combined with HRGrad, and we employ it as the default strategy throughout all experiments.

\subsection{Boltzmann-BGK problems}
Consider the Boltzmann-BGK equation
\begin{equation}
    \partial_t f + v \cdot \nabla_x f = \frac{1}{\varepsilon} \left ( M(U) - f \right ),  \quad v \in \mathbb{R}.
\end{equation}
and recall that the system of the BGK model for APNNs is
\begin{equation}
    \left \{
    \begin{aligned}
         & \varepsilon \left ( \partial_t f + v \partial_x f \right ) = M(U) - f, \\
         & \partial_t U
        + \nabla_x \cdot \left \langle v m f \right \rangle = 0,                  \\
         & U = \left \langle m f \right \rangle .
    \end{aligned}
    \right.
\end{equation}
Initial thermodynamic equilibrium is considered for all tests, as follows
\begin{equation}
    f(0, x, v) = \frac{\rho_0}{{(2 \pi T_0)}^{\frac{1}{2}}} \exp \left ( - \frac{|v - u_0|^2}{2 T_0} \right ) := f_0 (x, v).
\end{equation}

For these tests, a classical Riemann-like problem with Sod like initial velocity data is considered, as
\begin{equation}
    \begin{aligned}
         & \rho_0(t, x) = \rho_L, & \quad u_0(t, x) = u_L, & \quad T_0(t, x) = T_L, & \quad x \leq 0, \\
         & \rho_0(t, x) = \rho_R, & \quad u_0(t, x) = u_R, & \quad T_0(t, x) = T_R, & \quad x > 0.
    \end{aligned}
\end{equation}

Furthermore, HRGrad encodes the density distribution variable $f_{\theta}$ and the macro variables $\rho_{\theta}, u_{\theta}, T_{\theta}$ using two separate channels and decodes them concurrently in the base task. 
In this problem, we impose hard constraints on the predictions to ensure that $f_{\theta}$ remains positive and that $\rho_{\theta}, u_{\theta}, T_{\theta}$ automatically satisfy the boundary constraints. 
The formulation is as follows
\begin{equation}
\left \{
\begin{aligned}
     &f_{\theta}(t, x, v) := \ln \left(1 + \exp (\tilde{f}_{\theta}(t, x, v)) \right) > 0, \\
     & \rho_{\theta}(t, x) := \rho_L^{\frac{x_R - x}{x_R - x_L}} \cdot \rho_R^{\frac{x - x_L}{x_R - x_L}} \cdot \exp \left ((x-x_L)(x_R - x) \cdot  \tilde{\rho}_{\theta}(t, x)\right ),\\
     & u_{\theta}(t, x) := \sqrt{(x-x_L)(x_R - x)} \cdot \tilde{u}_{\theta}(t, x) + \frac{x_R - x}{x_R - x_L} u_L + \frac{x - x_L}{x_R - x_L} u_R,  \\
     &T_{\theta}(t, x) := T_L^{\frac{x_R - x}{x_R - x_L}} \cdot T_R^{\frac{x - x_L}{x_R - x_L}} \cdot \exp \left ( (x-x_L)(x_R - x) \cdot  \tilde{T}_{\theta}(t, x)\right ). \\
\end{aligned}
\right.
\end{equation}

\subsubsection{1D smooth problems}\label{sec:1d_smooth}
In this section, we perform a test on the 1D smooth problem \cite{li2024solving} for $\varepsilon \in (0,1]$. 
The time interval for the test is set to $t \in (0, 0.1]$. 
The spatial domain is defined as $x \in (-0.5, 0.5)$ with periodic boundary conditions applied. 
For the microscopic velocity space, the computational domain is chosen as $v \in (-10, 10)$. 
The initial condition is specified by a Maxwellian distribution with particular macroscopic moments, as
\begin{equation}
    \rho_0(x) = 1 + 0.5\sin(2\pi x), \quad
    u_0(x) = 0, \quad 
    T_0(x) = 1 + 0.5\sin(2\pi x + 0.2).
\end{equation}

In this experiment, we employ the trapezoidal rule with $257$ points for the numerical evaluation of macroscopic moments.

\subsubsection{1D Riemann problems}
In this section, we conduct a test on the one-dimensional Sod tube problem \cite{li2024solving} for $\varepsilon \in (0,1]$. 
The computational settings for this test largely mirror those used in \Cref{sec:1d_smooth}, with the exceptions being the boundary condition, the number of collocation points, and the initial condition.
For the spatial domain, we implement the free-flow boundary condition.
The initial condition is represented by a local Maxwellian distribution with macroscopic moments $(\rho_0, u_0, T_0)$. 
Specifically, the values on the interval $(-0.5, 0)$ are set to $(1, 0, 1)$, and on $[0, 0.5)$, they are $(0.125, 0, 0.8)$. 
Given the challenge of approximating jump functions with neural networks, we opt for a smoothed version of the macroscopic moments
\begin{equation}
    \rho_0(x) = 1.5 + (0.625-1.5)H(x), \quad
    u_0(x) = 0, \quad
    T_0(x) = 1.5 + (0.75-1.5)H(x).
\end{equation}
where, $H(x)=\left(1 + \tanh(20x)\right)/2$ serves as a smoothed approximation of the Heaviside function.

\begin{figure}[!htb]
     \centering
     \includegraphics[width=\textwidth]{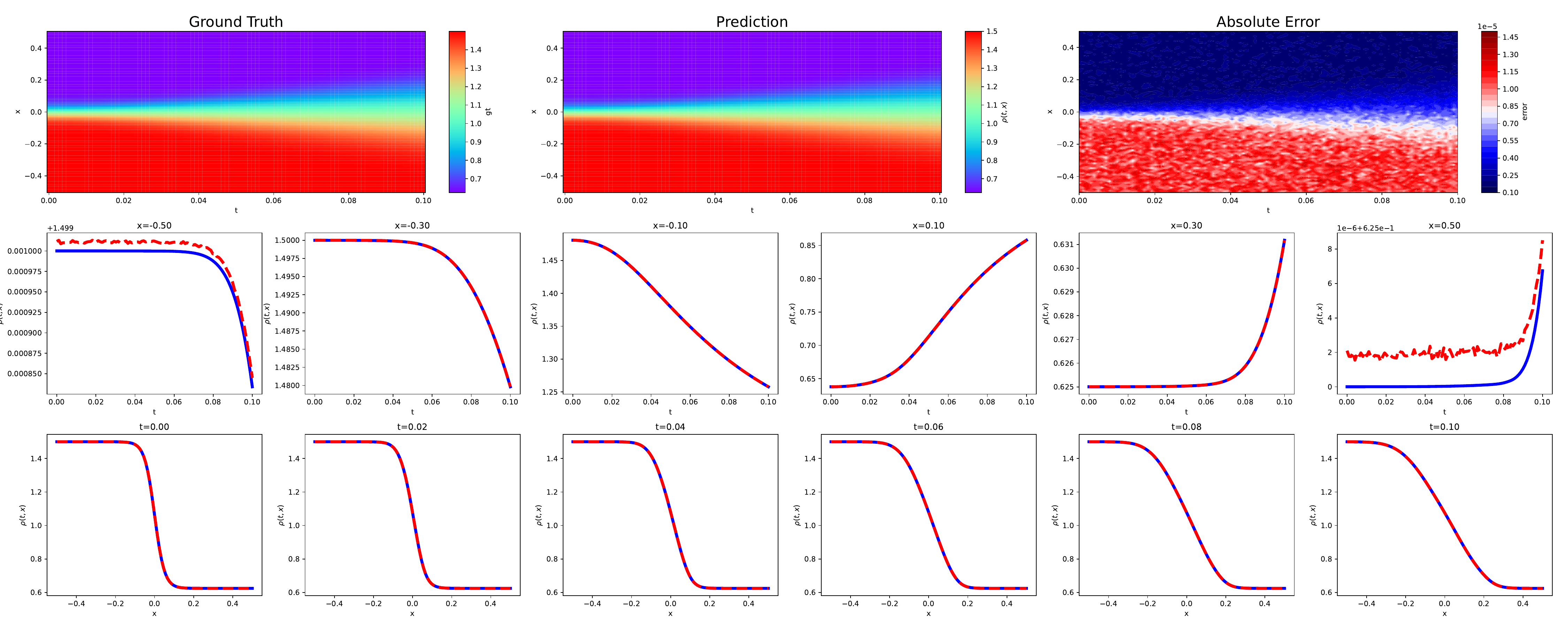}
     \vspace{-16pt}
     \caption{HRGrad predicts the macroscopic density moments $\rho$ of the 1D Riemann problem for $\varepsilon = 0.1$.}
\end{figure}

\begin{figure}[!htb]
     \centering
     \includegraphics[width=\textwidth]{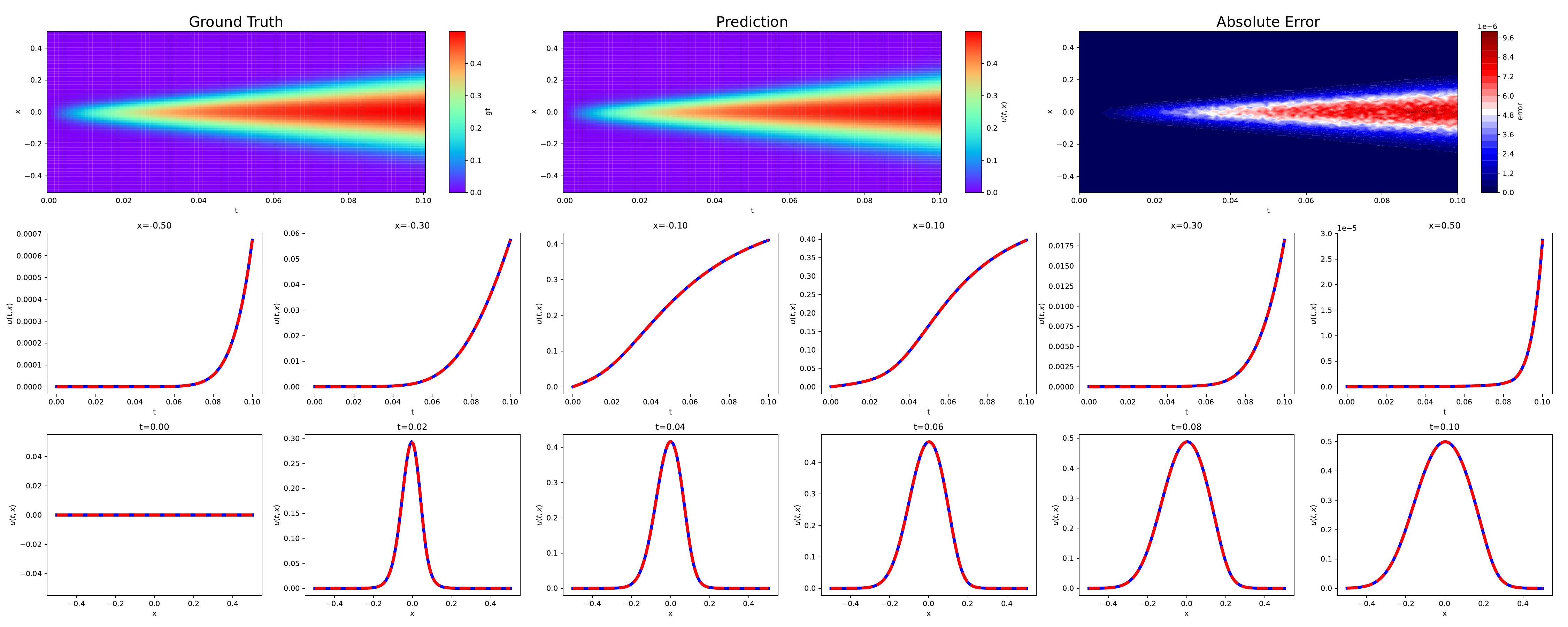}
     \vspace{-16pt}
     \caption{HRGrad predicts the macroscopic velocity moments $u$ of the 1D Riemann problem for $\varepsilon = 0.1$.}
\end{figure}

\begin{figure}[!htb]
     \centering
     \includegraphics[width=\textwidth]{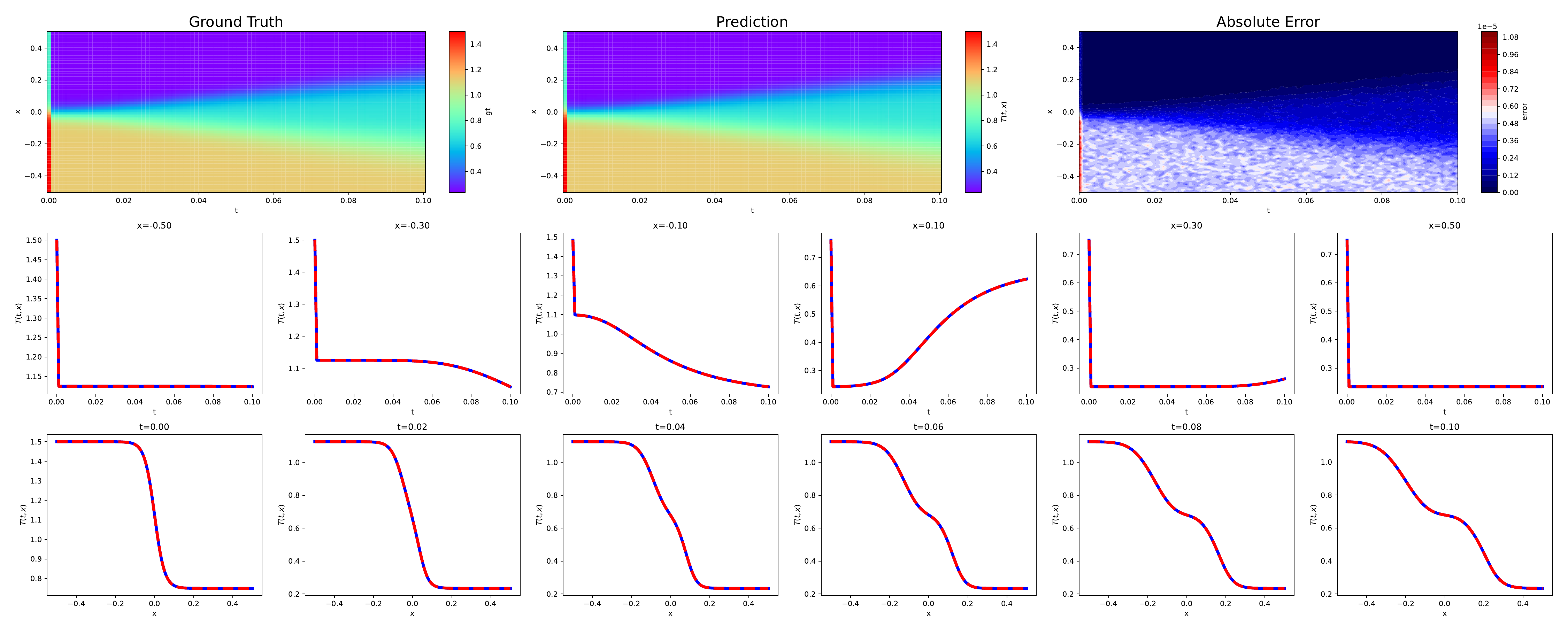}
     \vspace{-16pt}
     \caption{HRGrad predicts the macroscopic velocity moments $T$ of the 1D Riemann problem for $\varepsilon = 0.1$.}
\end{figure}

\subsubsection{ES-BGK problems}

To further examine the performance of HRGrad beyond the classical BGK relaxation model, we consider the ellipsoidal statistical BGK (ES-BGK) equation.
Compared with the standard BGK model, the ES-BGK equation introduces an anisotropic Gaussian equilibrium depending on the stress tensor, which enables the model to recover the correct Prandtl number in the Navier--Stokes--Fourier asymptotic regime.
This makes the ES-BGK problem a natural higher-level benchmark for testing multiscale optimization methods, since the neural solver must simultaneously resolve the kinetic distribution, the macroscopic conservation laws, and the anisotropic moment structure.

We consider the one-dimensional-in-space and two-dimensional-in-velocity ES-BGK equation
\begin{equation}
    \partial_t f + v_1 \partial_x f
    =
    \frac{\tau(\rho,T)}{\varepsilon}
    \left( G[f] - f \right),
    \qquad
    x \in \Omega \subset \mathbb{R},
    \quad v=(v_1,v_2)\in\mathbb{R}^2 ,
\end{equation}
where $f=f(t,x,v)$ is the particle distribution function, $\varepsilon$ is the Knudsen number, and $G[f]$ is the ellipsoidal Gaussian equilibrium.
The macroscopic quantities are defined by
\begin{equation}
    \rho = \int_{\mathbb{R}^2} f\,\diff v,
    \qquad
    \rho u = \int_{\mathbb{R}^2} v f\,\diff v,
    \qquad
    E = \frac{1}{2}\int_{\mathbb{R}^2} |v|^2 f\,\diff v .
\end{equation}
The temperature tensor is given by
\begin{equation}
    \Theta
    =
    \frac{1}{\rho}
    \int_{\mathbb{R}^2}
    (v-u)\otimes(v-u) f\,\diff v ,
\end{equation}
and the corrected tensor temperature is
\begin{equation}
    \mathcal{T}
    =
    (1-\nu) T I + \nu \Theta,
    \qquad
    -\frac{1}{2} \leq \nu < 1 .
\end{equation}
The ellipsoidal Gaussian is then defined as
\begin{equation}
    G[f]
    =
    \frac{\rho}{\sqrt{\det(2\pi\mathcal{T})}}
    \exp\left(
        -\frac{1}{2}
        (v-u)^\top \mathcal{T}^{-1}(v-u)
    \right).
\end{equation}
The parameter $\nu$ controls the Prandtl number through
\begin{equation}
    \mathrm{Pr} = \frac{1}{1-\nu}.
\end{equation}

In the hydrodynamic limit $\varepsilon\to 0$, the solution approaches a local equilibrium and the leading-order macroscopic dynamics are governed by the compressible Euler system.
When the first-order Chapman--Enskog correction is taken into account, the ES-BGK model is consistent with the Navier--Stokes--Fourier equations with an adjustable Prandtl number.
Therefore, this example allows us to test whether HRGrad can handle not only the kinetic-to-fluid transition, but also the optimization conflict induced by anisotropic stress and heat-flux corrections.

We adopt a Riemann-type benchmark with left and right macroscopic states
\begin{equation}
    (\rho_L,u_L,T_L)=(1,M\sqrt{2},1),
    \qquad
    (\rho_R,u_R,T_R)=(1,0,1.05),
\end{equation}
where $M=2.5$.
The initial distribution is taken as the corresponding local Maxwellian generated from these macroscopic states.
The computational domain is chosen as
\begin{equation}
    x\in[-0.5,0.5],
    \qquad
    v\in[-V,V]^2,
\end{equation}
with $V=8$ unless otherwise specified.
The representative Knudsen numbers are selected as
\begin{equation}
    \varepsilon \in \left\{5\times 10^{-1},10^{-1},10^{-2},10^{-3}\right\},
\end{equation}
and the multiscale training version samples $\varepsilon$ from the same log-uniform distribution used in the other experiments.

In the neural approximation, we use separate network channels for the distribution function $f_\theta$ and the macroscopic variables $U_\theta=(\rho_\theta,u_\theta,T_\theta)$.
To improve stability, the positivity of $f_\theta$, $\rho_\theta$, and $T_\theta$ is enforced through softplus or exponential parameterizations.
Since the ES-BGK equilibrium depends on the tensor $\mathcal{T}$, we additionally impose a stress-consistency loss between the stress tensor computed from $f_\theta$ and the auxiliary macroscopic stress representation.
The loss contains the kinetic residual, the conservation-law residual, the moment-consistency residual, the ES equilibrium residual, and the initial and boundary losses.
These loss components are treated as different tasks and are aggregated by HRGrad.

The main diagnostic quantities are the density $\rho$, velocity $u$, temperature $T$, heat flux
\begin{equation}
    Q_1(t,x)
    =
    \frac{1}{\varepsilon}
    \int_{\mathbb{R}^2}
    \frac{|v-u|^2}{2}(v_1-u_1) f\,\diff v,
\end{equation}
and the anisotropy indicator
\begin{equation}
    \mathcal{A}(t,x)
    =
    \frac{\|\Theta - T I\|_F}{T}.
\end{equation}
For large $\varepsilon$, the HRGrad prediction is compared with a high-resolution discrete-velocity ES-BGK solver.
For small $\varepsilon$, we also compare the macroscopic quantities with the corresponding fluid-limit reference solution.
This setup provides a more stringent test than the standard BGK equation, because the method must preserve both the stiff relaxation structure and the anisotropic moment information across different kinetic regimes.

\subsection{Linear transport problems}
We consider representative slab-geometry transport benchmarks from the rarefied regime ($\varepsilon \approx O(1)$) to the diffusive regime ($\varepsilon \to 0$). 
To keep the presentation focused, we organize them into three groups: a smooth periodic benchmark, a family of one-dimensional non-periodic stress tests, and a two-dimensional smooth benchmark.
Together they cover different initial-boundary conditions, heterogeneous coefficients, boundary-layer effects, and dimensional extension.

Consider the linear transport equation
\begin{equation}
    \varepsilon \partial_t f + v\partial_x f = \frac{1}{\varepsilon}\left ( \frac{1}{2} \int_{-1}^1 f \, \diff v' - f \right ),
\end{equation}
and recall that the even-odd decomposition of this equation is
\begin{equation}
\left \{
\begin{aligned}
     & \varepsilon^2 \partial_t r + \varepsilon^2 v\partial_x j = \rho - r, \\
     & \varepsilon^2 \partial_t j + v \partial_x r = - j,                   \\
     & \partial_t \rho +  \left \langle v \partial_x j \right \rangle = 0,  \\
     & \rho = \left \langle r \right \rangle.                               \\
\end{aligned}
\right.
\end{equation}

\subsubsection{Smooth initial data with periodic BC}
We start from the rarefied regimes where $\varepsilon=1$ to the diffusive regime where $\varepsilon\to 0$, and consider periodic boundary conditions with a smooth initial data as follows
\begin{equation}
    f_0(x, v) = \frac{\rho(x)}{\sqrt{2\pi}}e^{-\frac{v^2}{2}},
\end{equation}
where
\begin{equation}
    \rho(x) = 1 + \cos (4 \pi x).
\end{equation}

The source term, scattering, and absorbing coefficients are set as
\begin{equation}
    \sigma_S = 1, \quad \sigma_A = 0, \quad Q = 0, \quad \varepsilon\in(0,1].
\end{equation}

To enhance numerical performance, exact periodic boundary conditions are enforced. 
The approach relies on a Fourier basis, where the transform
\begin{equation}
    T: x \to \{\sin( 2\pi j x), \cos(2\pi j x)\}_{j = 1}^k
\end{equation}
is applied prior to the first layer of the DNNs, as outlined in \cite{han2020sol, lyu2020enforcing}. 
In this paper, we set $k = 8$ to strike a balance between computational efficiency and the accuracy of the Fourier representation.

\subsubsection{1D non-periodic transport problems}
We next collect three representative one-dimensional non-periodic tests, which stress HRGrad under boundary-driven transport, spatially varying scattering, and kinetic boundary layers.

\paragraph{Isotropic in-flow boundary condition}
We first consider isotropic in-flow boundary conditions given by
\begin{equation}
    x \in [0, 1], \quad F_L(v) = 1, \quad F_R(v) = 0.
\end{equation}
The initial condition is set as $f_0(x, v) = 0$.
The source term, scattering, and absorbing coefficients are defined as
\begin{equation}
    Q = 0, \quad \sigma_S = 1, \quad \sigma_A = 0, \quad \varepsilon\in(0,1],
\end{equation}
The results are illustrated in Figure, where the exact value of $\left \langle g \right \rangle = 0$ is achieved. 

To enhance numerical performance, $\rho_{\theta}$ can be further designed to inherently satisfy the initial condition, which is given by
\begin{equation}
    \rho_{\theta}(t, x) := t \cdot \exp \left( -\tilde{\rho}_{\theta}(t, x)\right) \approx \rho(t, x).
\end{equation}

\paragraph{Variable scattering coefficient}
Let
\begin{equation}
    x \in [0, 1], \quad F_L(v) = 1, \quad F_R(v) = 0,
\end{equation}
{and initial condition $f_0(x, v) = 0$.} The source term, scattering, and absorbing coefficients are set as
\begin{equation}
    Q = 1, \quad \sigma_S = 1 + (10x)^2, \quad \sigma_A = 0, \quad \varepsilon\in(0,1].
\end{equation}
In Figure, we report the numerical solution at time $t=0.0, 0.1, 0.2$.
In this problem, we have a source term and the scattering cross section that depend on $x$, so the scaling term $\sigma_S / \varepsilon$ ranges from $1/\varepsilon \to O(1)$, a problem with mixing scales.
The numerical results show reasonably good performance of the proposed HRGrad method.

\paragraph{Boundary-layer case}
Let
\begin{equation}
    x \in [0, 1], \quad F_L(v) = 5 \sin(v), \quad F_R(v) = 0,
\end{equation}
and initial condition $f_0(x, v) = 0$.
The source term, scattering, and absorbing coefficients are set as
\begin{equation}
    Q = 0, \quad \sigma_S = 1, \quad \sigma_A = 0, \quad \varepsilon\in(0,1].
\end{equation}
where since $F_L$ depends on $v$, there is a boundary layer near $x=0$.

\subsubsection{2D smooth problems}
Consider a 2D smooth problem from the Rarefied regime to the diffusive regime with
\begin{equation}
    \Gamma =  [0, 1] \times [0, 1], \quad F_B(x,v) = 0, \quad n \cdot v < 0, \quad x \in \partial \Gamma,
\end{equation}
and initial condition $f_0(x, v) = 0$.
The source term, scattering, and absorbing coefficients are set as
\begin{equation}
    Q = 1, \quad \sigma_S = 1, \quad \sigma_A = 0, \quad  \varepsilon\in(0,1].
\end{equation}
were $n$ denotes the exterior unit normal vector on $\partial \Gamma$.

 

\subsubsection{Linear semiconductor Boltzmann--Poisson problems}

We next consider a linear semiconductor Boltzmann--Poisson problem, which can be viewed as a natural extension of the linear transport equation with a self-consistent electric field.
Compared with the pure linear transport benchmarks, this problem contains an additional velocity-space drift term, a Poisson equation for the electrostatic potential, contact boundary layers, and a drift--diffusion asymptotic limit.
It therefore provides a more challenging multiscale test for HRGrad.

We consider the one-dimensional-in-space and one-dimensional-in-velocity model
\begin{equation}
    \varepsilon \partial_t f
    + v\partial_x f
    - E(t,x)\partial_v f
    =
    \frac{1}{\varepsilon} Q(f),
    \qquad
    x\in(0,1), \quad v\in\mathbb{R},
\end{equation}
where $f=f(t,x,v)$ is the electron distribution function, $E$ is the electric field, and $Q(f)$ is a linear relaxation operator.
In the numerical experiments, we use the relaxation-time approximation
\begin{equation}
    Q(f)=\rho M-f,
    \qquad
    \rho(t,x)=\int_{\mathbb{R}} f(t,x,v)\,\diff v,
\end{equation}
where $M(v)$ is a normalized Maxwellian.
The electric potential $\phi$ and the electric field $E$ are determined by the Poisson equation
\begin{equation}
    -\beta \partial_{xx}\phi = \rho - c(x),
    \qquad
    E=-\partial_x\phi ,
\end{equation}
where $\beta$ denotes the scaled Debye length and $c(x)$ is the doping profile.

In the diffusive scaling, as $\varepsilon\to 0$, the distribution approaches local equilibrium,
\begin{equation}
    f = \rho M + O(\varepsilon),
\end{equation}
and the limiting macroscopic model is the drift--diffusion--Poisson system
\begin{equation}
    \partial_t \rho
    =
    \partial_x
    \left(
        D\partial_x\rho
        + \eta \rho E
    \right),
    \qquad
    -\beta \partial_{xx}\phi = \rho - c(x),
    \qquad
    E=-\partial_x\phi .
\end{equation}
Here $D$ is the diffusion coefficient and $\eta$ is the mobility coefficient.
This asymptotic structure is analogous to the diffusion limit of the linear transport equation, but the additional electric-field coupling introduces stronger task interaction and more severe optimization stiffness.

We use a smooth $n^+-n-n^+$ diode-type doping profile,
\begin{equation}
\begin{aligned}
    c(x)
    =
    C_n
    &+
    \frac{C_+-C_n}{2}
    \left[
        1+\tanh\left(\frac{x_1-x}{s}\right)
    \right]  \\
    &+
    \frac{C_+-C_n}{2}
    \left[
        1+\tanh\left(\frac{x-x_2}{s}\right)
    \right],
\end{aligned}
\end{equation}
where $x_1=0.3$, $x_2=0.7$, and $s=0.02$.
Unless otherwise stated, we set
\begin{equation}
    \beta = 2\times 10^{-3},
    \qquad
    V_b = 5,
\end{equation}
and impose the Dirichlet boundary condition for the electric potential
\begin{equation}
    \phi(t,0)=0,
    \qquad
    \phi(t,1)=V_b .
\end{equation}
For the kinetic distribution, the inflow boundary conditions are prescribed as
\begin{equation}
    f(t,0,v)=F_L(v), \quad v>0,
    \qquad
    f(t,1,v)=F_R(v), \quad v<0,
\end{equation}
where $F_L$ and $F_R$ are Maxwellian inflow data associated with the contact densities.
The initial condition is chosen as
\begin{equation}
    f(0,x,v)=c(x)M(v).
\end{equation}

The representative Knudsen numbers are
\begin{equation}
    \varepsilon\in\{1,10^{-1},10^{-2},10^{-3}\}.
\end{equation}
For the multiscale version, $\varepsilon$ is sampled from the log-uniform distribution on $[\varepsilon_{\min},1]$, as described in the multiscale sampling paragraph.
This is particularly important for the Boltzmann--Poisson problem, since the small-$\varepsilon$ regime is dominated by the drift--diffusion balance and the Poisson coupling, whereas the kinetic regime requires resolving the velocity-space non-equilibrium induced by the electric field.

For the neural representation, we use a micro--macro form
\begin{equation}
    f_\theta(t,x,v)
    =
    \rho_\theta(t,x)M(v)
    +
    \varepsilon g_\theta(t,x,v),
    \qquad
    \left\langle g_\theta \right\rangle = 0,
\end{equation}
together with an additional potential head $\phi_\theta(t,x)$.
The electric field is then computed by
\begin{equation}
    E_\theta(t,x)=-\partial_x\phi_\theta(t,x).
\end{equation}
The total loss contains the kinetic residual, the micro--macro consistency residual, the Poisson residual, the drift--diffusion limiting residual, and the initial and boundary losses.
The Poisson residual is essential in this example, because an inaccurate electric field can lead to a wrong drift direction even if the density error is small.
HRGrad is then applied to aggregate the gradients associated with the different $\varepsilon$-tasks and different physical residuals.

The reference solutions are generated by a deterministic asymptotic-preserving solver based on parity or micro--macro decomposition, with Gauss--Hermite quadrature in velocity and finite differences in space.
For $\varepsilon=O(1)$, we compare the predicted density $\rho$, current $J$, electric field $E$, and potential $\phi$ with the kinetic reference solution.
For $\varepsilon\ll 1$, we additionally compare with the drift--diffusion--Poisson limit.
The current is computed by
\begin{equation}
    J(t,x)=\int_{\mathbb{R}} v f(t,x,v)\,\diff v.
\end{equation}
The main error metrics are the relative $\ell^2$ errors of $\rho$, $E$, and $\phi$, the Poisson residual, the mass conservation error, and the asymptotic-preserving error
\begin{equation}
    e_{\mathrm{AP}}(t)
    =
    \frac{\|f_\theta-\rho_\theta M\|_{L^1_{x,v}}}
    {\|f_\theta\|_{L^1_{x,v}}}.
\end{equation}
This example directly tests whether HRGrad can maintain stable training when the kinetic equation, the macroscopic limit equation, and the self-consistent field equation impose competing gradient directions.

For the additional ES-BGK and semiconductor Boltzmann--Poisson experiments, the same multiscale training protocol is adopted.
In particular, $\varepsilon$ is treated as a task parameter and sampled by the log-uniform strategy.
The loss functions are decomposed into kinetic residuals, macroscopic limiting residuals, moment or field consistency residuals, and initial-boundary residuals, so that HRGrad can explicitly correct the gradient conflict among different scales and physical constraints.

\section{Conclusions}\label{sec:conclusion}

In this study, we have developed a harmonized rotational gradient method (HRGrad) for training the model to simultaneously solve both the Boltzmann-BGK equation and the linear transport equation across all ranges of Knudsen numbers.
This is also a multi-scale problem where the Knudsen number gradually decreases, reflecting a transition from large-scale to small-scale phenomena, leading to a shift from free molecular flow to a continuum flow state.
The computational challenge inherent in this task arises from the multi-scale nature of the kinetic equation and the simultaneous handling of different scale-solving issues. 
Existing methods for addressing such problems often experience gradient conflicts, making effective training difficult. 
To address this, we propose the harmonized rotational gradient method (HRGrad), which not only corrects the gradient magnitude but also adjusts the gradient direction, balancing the task losses for different Knudsen numbers by ensuring each task's loss decreases at approximately the same rate.
Additionally, we employ APNNs to solve these equations using macro-micro and even-odd decomposition for each task, addressing the small-scale challenges for PINNs.
Through a series of numerical tests, we demonstrate that HRGrad is capable of predicting solutions for different Knudsen numbers of the Boltzmann-BGK equation and the linear transport equation, even for unseen Knudsen numbers.

In addition, there is still room for improvement on the complexity of the proposed HRGrad method for solving high-dimensional kinetic problems, along with the computational cost of handling complex collisions and high-dimensional integrals, which we leave for future work.
In future work, we aim to integrate a neural sparse representation method \cite{li2024solving} to efficiently approximate the Maxwellian distribution within the BGK and quadratic collision models, thereby enhancing HRGrad's computational efficiency for each task. 
Furthermore, we propose incorporating a separable physics-informed neural network (SPINN)-based method \cite{oh2025separable} to address high-dimensional Boltzmann-BGK problems, enabling HRGrad to handle even more intricate high-dimensional challenges effectively.

\section*{Acknowledgments}
We would like to thank the referees for their insightful comments that greatly improved the paper.
The computing for this work was supported by the High Performance Computing Platform at Eastern Institute of Technology, Ningbo.

\appendix
\section{Convergence analysis}
\label{app:convergence_analysis}

In this appendix, we provide the formal proof of global convergence for both convex and
non-convex landscapes for the proposed HRGrad method, together with two supplementary
analyses: \textit{per-task gradient descent compatibility} (how the isometric rotation
preserves individual task optimization) and \textit{MCAE angular stability}
(how the MER objective achieves the minimal necessary rotation for conflict resolution).
All four HRGrad stages contribute explicitly to the results below.

\textbf{Notation.}  Throughout this appendix, $g_i = \nabla_\theta \mathcal{L}_i(\theta)$ denotes
the $i$-th task gradient, $\bar{g}_i = \mathcal{U}(g_i)$ its unit vector,
$\alpha_i^* \in [0,\pi/2]$ the optimal rotation angle from the MER objective (Eq.~\eqref{eq:mer_objective}),
and $g_i^{\mathrm{rot}}$ the isometrically rotated counterpart defined in Eq.~\eqref{eq:grot_def}.
We write $\mathbb{K} = \{G\lambda \mid \lambda \in \mathbb{R}_+^m\}$ for the Primal Gradient Cone,
$\mathbb{K}^* = \{y \in \mathbb{R}^D \mid G^\top y \geq \mathbf{0}_m\}$ for the Dual Gradient Cone,
and $\mathbb{H} = \mathbb{K} \cap \mathbb{K}^*$ for the Harmonized Cone.

The convergence analysis rests on two geometric properties of HRGrad that follow
directly from Lemma~\ref{app:lem:hrgrad} and one regularity condition on the
alignment between the original task gradients and the fair consensus direction.

\begin{description}
  \item[(P1) Equal cosine similarity\quad(Lemma~\ref{app:lem:hrgrad}\textit{(iii)})]
    There exists a constant $S_c > 0$ such that
    $\mathcal{U}(g_i^{\mathrm{rot}})^\top g_u = S_c$ for all $i = 1,\ldots,m$.
  \item[(P2) Magnitude conservation\quad(Lemma~\ref{app:lem:hrgrad}\textit{(i),(iii)})]
    $\|g_i^{\mathrm{rot}}\| = \|g_i\|$ for all $i$, and consequently
    $\|g_{\mathrm{HRGrad}}\| = S_c\sum_{i=1}^m\|g_i\|$.
\end{description}

\begin{assumption}[Strict Common Descent]\label{ass:scd}
  At every iterate $\theta$, the original task gradients satisfy
  $g_i^\top g_u \geq \rho_i\|g_i\|$ for some constants $\rho_i > 0$,
  $i = 1,\ldots,m$.  Set $\rho_{\min} = \min_{1\leq i\leq m}\rho_i > 0$.
\end{assumption}

\begin{remark}[Sufficient condition for Assumption~\ref{ass:scd}]\label{rem:scd}
  Let $\alpha_{\max}^* = \max_i \alpha_i^*$ be the largest MER rotation angle.
  The chord-length identity $\|g_i^{\mathrm{rot}} - g_i\| = 2\|g_i\|\sin(\alpha_i^*/2)$
  and the Cauchy--Schwarz inequality give
  \begin{equation}\label{eq:app:scd_suff}
    \sum_{i=1}^m g_i^\top g_u
    = \sum_{i=1}^m (g_i^{\mathrm{rot}})^\top g_u
      - \sum_{i=1}^m (g_i^{\mathrm{rot}} - g_i)^\top g_u
    \geq \Bigl(S_c - 2\sin\tfrac{\alpha_{\max}^*}{2}\Bigr)
         \sum_{i=1}^m \|g_i\|,
  \end{equation}
  so Assumption~\ref{ass:scd} holds with
  $\rho_{\min} = S_c - 2\sin(\alpha_{\max}^*/2) > 0$
  whenever $\alpha_{\max}^* < 2\arcsin(S_c/2)$.
  In the conflict-free case $\alpha_i^* = 0$, this gives $\rho_{\min} = S_c$ and
  $\kappa = 1$ (see Corollary~\ref{cor:conflict_free}).
\end{remark}

Under Assumption~\ref{ass:scd}, we define the \emph{manifold mapping coefficient}
\begin{equation}\label{eq:app:kappa}
  \kappa := \frac{\rho_{\min}}{S_c} \in (0, 1],
\end{equation}
which equals $1$ precisely when all task gradients are non-conflicting.

\subsection{Complete Proof of the Geometric Lemma}\label{app:proof:lemma}

\begin{lemma}[Geometric Properties of HRGrad]\label{app:lem:hrgrad}
Let $\{g_i^{\mathrm{rot}}\}_{i=1}^m$ and $g_{\mathrm{HRGrad}}$ be defined in
Eqs.~\eqref{eq:grot_def}, \eqref{eq:HRGrad_gu}--\eqref{eq:HRGrad_final},
and let $M = [\mathcal{U}(g_1^{\mathrm{rot}}), \ldots, \mathcal{U}(g_m^{\mathrm{rot}})] \in \mathbb{R}^{D \times m}$
have full column rank.  Then:
\begin{itemize}
  \item[\textit{(i)}] \textbf{(Isometric Fidelity)}
    $\|g_i^{\mathrm{rot}}\| = \|g_i\|$ for all $i$.
  \item[\textit{(ii)}] \textbf{(Non-Conflict)}
    $(g_i^{\mathrm{rot}})^\top g_{\mathrm{HRGrad}} \geq 0$ for all $i$,
    with strict inequality when $\|g_{\mathrm{HRGrad}}\| > 0$.
  \item[\textit{(iii)}] \textbf{(Equal Cosine Similarity)}
    $\exists\, S_c > 0$ s.t.\ $\dfrac{(g_i^{\mathrm{rot}})^\top g_{\mathrm{HRGrad}}}
    {\|g_i^{\mathrm{rot}}\|\,\|g_{\mathrm{HRGrad}}\|} = S_c$ for all $i$.
  \item[\textit{(iv)}] \textbf{(Aggregate Product Identity)}
    $\bigl(\sum_{i=1}^m g_i^{\mathrm{rot}}\bigr)^\top g_{\mathrm{HRGrad}} = \|g_{\mathrm{HRGrad}}\|^2$.
\end{itemize}
\end{lemma}

\begin{proof}
\textit{(i)} $[\text{Stage 2 isometry}]$\quad
For a conflicting task ($\bar{g}_i^\top d^* < 0$): the Gram-Schmidt step ensures
$\bar{g}_i \perp w_i$ and $\|\bar{g}_i\| = \|w_i\| = 1$, so the rotation preserves unit length:
\begin{equation}\label{eq:app:isom}
  \|r_i(\alpha_i^*)\|^2 = \cos^2(\alpha_i^*)\|\bar{g}_i\|^2 + \sin^2(\alpha_i^*)\|w_i\|^2
  = \cos^2(\alpha_i^*) + \sin^2(\alpha_i^*) = 1.
\end{equation}
The Stage~4 magnitude restoration gives $g_i^{\mathrm{rot}} = \|g_i\|\,r_i(\alpha_i^*)$,
so $\|g_i^{\mathrm{rot}}\| = \|g_i\|$.
For a non-conflicting task ($\bar{g}_i^\top d^* \geq 0$): $g_i^{\mathrm{rot}} = g_i$
by definition~\eqref{eq:grot_def}, so the identity is trivially satisfied.

\textit{(ii)} $[\text{Harmonized Cone dual property}]$\quad
By Stages~1--3, every rotated gradient $g_i^{\mathrm{rot}}$ is deflected into
$\mathbb{H} = \mathbb{K} \cap \mathbb{K}^*$.
Since $g_u$ is constructed as the normalized centroid of the extreme rays of $\mathbb{H}$
(Stage~4), we have $g_u \in \mathrm{Int}(\mathbb{H}) \subseteq \mathbb{K}^*$.
By the definition of the Dual Cone $\mathbb{K}^*$:
\begin{equation}\label{eq:app:dual}
  g_u \in \mathbb{K}^* \;\Longrightarrow\; (g_i^{\mathrm{rot}})^\top g_u \geq 0
  \quad\forall\;i \text{ (since $g_i^{\mathrm{rot}} \in \mathbb{K}$).}
\end{equation}
From Eq.~\eqref{eq:HRGrad_final}, $g_{\mathrm{HRGrad}} = \bigl(\sum_j (g_j^{\mathrm{rot}})^\top g_u\bigr) g_u$:
\begin{equation}\label{eq:app:noncf}
  (g_i^{\mathrm{rot}})^\top g_{\mathrm{HRGrad}}
  = \underbrace{\Bigl(\sum_{j=1}^m (g_j^{\mathrm{rot}})^\top g_u\Bigr)}_{\geq\,0}
    \cdot \underbrace{(g_i^{\mathrm{rot}})^\top g_u}_{\geq\,0} \geq 0.
\end{equation}
Equality holds iff $\sum_j (g_j^{\mathrm{rot}})^\top g_u = 0$, i.e., $\|g_{\mathrm{HRGrad}}\| = 0$.

\textit{(iii)} $[\text{Pseudo-inverse fair aggregation}]$\quad
Since $M$ has full column rank, the Moore-Penrose pseudoinverse satisfies
$(M^\dagger)^\top = M(M^\top M)^{-1}$, so
$g_u \propto M(M^\top M)^{-1}\mathbf{1}_m$.
Computing the projected product:
\begin{equation}\label{eq:app:sc}
  M^\top g_u
  = \frac{M^\top M(M^\top M)^{-1}\mathbf{1}_m}{\bigl\|M(M^\top M)^{-1}\mathbf{1}_m\bigr\|}
  = \frac{\mathbf{1}_m}{\bigl\|M(M^\top M)^{-1}\mathbf{1}_m\bigr\|}
  =: S_c\,\mathbf{1}_m, \quad S_c > 0.
\end{equation}
Hence $\mathcal{U}(g_i^{\mathrm{rot}})^\top g_u = S_c$ uniformly for all $i$.
Because $g_{\mathrm{HRGrad}} \propto g_u$, the cosine similarity $(g_i^{\mathrm{rot}})^\top g_{\mathrm{HRGrad}}
/ (\|g_i^{\mathrm{rot}}\|\|g_{\mathrm{HRGrad}}\|) = S_c$ is constant across all tasks.

\textit{(iv)} $[\text{Summation identity}]$\quad
Using \textit{(iii)}, each cross-product is
$(g_i^{\mathrm{rot}})^\top g_{\mathrm{HRGrad}} = \|g_i^{\mathrm{rot}}\|\,\|g_{\mathrm{HRGrad}}\|\,S_c$.
Summing and using the magnitude formula $\|g_{\mathrm{HRGrad}}\|
= \bigl(\sum_i (g_i^{\mathrm{rot}})^\top g_u\bigr) = \sum_i \|g_i^{\mathrm{rot}}\| S_c$
(from $\|g_u\|=1$ and \textit{(iii)}):
\begin{equation}\label{eq:app:aggid}
  \Bigl(\sum_{i=1}^m g_i^{\mathrm{rot}}\Bigr)^\top g_{\mathrm{HRGrad}}
  = \sum_{i=1}^m \|g_i^{\mathrm{rot}}\|\,\|g_{\mathrm{HRGrad}}\|\,S_c
  = \|g_{\mathrm{HRGrad}}\| \cdot \underbrace{\sum_{i=1}^m \|g_i^{\mathrm{rot}}\| S_c}_{=\,\|g_{\mathrm{HRGrad}}\|}
  = \|g_{\mathrm{HRGrad}}\|^2.
\end{equation}
\end{proof}

\subsection{Convergence for Convex Losses}\label{app:proof:convex}

\begin{theorem}[Convex Convergence of HRGrad]\label{app:thm:convex}
  Suppose Assumption~\ref{ass:scd} holds.  Let each $\mathcal{L}_i \colon \mathbb{R}^D \to \mathbb{R}$
  be convex and continuously differentiable, and let the total gradient
  $g = \nabla_\theta\mathcal{L} = \sum_{i=1}^m \nabla_\theta\mathcal{L}_i$ be
  $L$-Lipschitz continuous with constant $L > 0$.
  Then for any step size $\gamma \leq 2\kappa/L$, the HRGrad update
  $\theta^+ = \theta - \gamma g_{\mathrm{HRGrad}}$ satisfies
  \begin{equation}\label{eq:app:convex_bound}
    \mathcal{L}(\theta^+)
    \leq \mathcal{L}(\theta)
    - \gamma\kappa\!\Bigl(1 - \frac{L\gamma}{2\kappa}\Bigr)\!\|g_{\mathrm{HRGrad}}\|^2
    \leq \mathcal{L}(\theta).
  \end{equation}
  In particular, $\{\mathcal{L}(\theta^k)\}$ is monotonically non-increasing, and the iterates
  converge to the global optimum of $\mathcal{L}$ or to a Pareto stationary point
  at which $\|g_{\mathrm{HRGrad}}\| = 0$.
\end{theorem}

\begin{proof}
Since $g$ is $L$-Lipschitz, the descent lemma gives
\begin{equation}\label{eq:app:quad}
  \mathcal{L}(\theta^+)
  \leq \mathcal{L}(\theta) + g^\top(\theta^+ - \theta)
  + \frac{L}{2}\|\theta^+ - \theta\|^2.
\end{equation}
Substituting $\theta^+ - \theta = -\gamma g_{\mathrm{HRGrad}}$:
\begin{equation}\label{eq:app:ca_temp1}
  \mathcal{L}(\theta^+)
  \leq \mathcal{L}(\theta)
  - \gamma g^\top g_{\mathrm{HRGrad}}
  + \frac{L\gamma^2}{2}\|g_{\mathrm{HRGrad}}\|^2.
\end{equation}
We establish a lower bound for $g^\top g_{\mathrm{HRGrad}}$.
Write $g_{\mathrm{HRGrad}} = \|g_{\mathrm{HRGrad}}\| g_u$ and compute:
\begin{align}
  g^\top g_{\mathrm{HRGrad}}
  &= \|g_{\mathrm{HRGrad}}\|\sum_{i=1}^m g_i^\top g_u \notag \\
  &\geq \|g_{\mathrm{HRGrad}}\|\,\rho_{\min}\sum_{i=1}^m\|g_i\|
    \qquad\text{(Assumption~\ref{ass:scd})} \label{eq:app:inner} \\
  &= \|g_{\mathrm{HRGrad}}\|\,\rho_{\min}\cdot\frac{\|g_{\mathrm{HRGrad}}\|}{S_c}
    \qquad\text{((P2))} \notag \\
  &= \kappa\|g_{\mathrm{HRGrad}}\|^2. \label{eq:app:kappa_bound}
\end{align}
Substituting \eqref{eq:app:kappa_bound} into \eqref{eq:app:ca_temp1}:
\begin{equation}
  \mathcal{L}(\theta^+)
  \leq \mathcal{L}(\theta) - \gamma\kappa\|g_{\mathrm{HRGrad}}\|^2
  + \frac{L\gamma^2}{2}\|g_{\mathrm{HRGrad}}\|^2
  = \mathcal{L}(\theta)
  - \gamma\kappa\!\Bigl(1 - \frac{L\gamma}{2\kappa}\Bigr)\!\|g_{\mathrm{HRGrad}}\|^2.
\end{equation}
For $\gamma \leq 2\kappa/L$ the factor $(1 - L\gamma/(2\kappa)) \geq 0$,
so $\mathcal{L}(\theta^+) \leq \mathcal{L}(\theta)$ with equality iff
$\|g_{\mathrm{HRGrad}}\| = 0$.
By (P2), $\|g_{\mathrm{HRGrad}}\| = 0$ iff $\nabla_\theta\mathcal{L}_i(\theta) = 0$
for all $i$, which is the global optimum by convexity.
\end{proof}

\begin{corollary}[Conflict-free case]\label{cor:conflict_free}
  If all task gradients are non-conflicting ($\alpha_i^* = 0$ for all $i$), then
  $g_i^{\mathrm{rot}} = g_i$, $\kappa = 1$, and the bound \eqref{eq:app:convex_bound}
  tightens to
  $\mathcal{L}(\theta^+) \leq \mathcal{L}(\theta) - \gamma(1 - L\gamma/2)\|g_{\mathrm{HRGrad}}\|^2$
  with step size $\gamma \leq 2/L$.
  Moreover, the Aggregate Product Identity (Lemma~\ref{app:lem:hrgrad}\textit{(iv)})
  gives the exact equality $g^\top g_{\mathrm{HRGrad}} = \|g_{\mathrm{HRGrad}}\|^2$,
  so no alignment assumption is needed in this case.
\end{corollary}

\subsection{Convergence for Non-Convex Losses}\label{app:proof:nonconvex}

\begin{theorem}[Non-Convex Convergence of HRGrad]\label{app:thm:nonconvex}
  Suppose Assumption~\ref{ass:scd} holds.  Let each $\mathcal{L}_i$ be continuously
  differentiable (not necessarily convex), and let
  $g = \nabla_\theta\mathcal{L}$ be $L$-Lipschitz continuous.  Define
  $\mathcal{L}^* = \inf_{\theta}\mathcal{L}(\theta) > -\infty$ and
  $\alpha = \min_{0\leq k\leq K-1} S_c^k > 0$.  Then for step size
  $\gamma \leq \kappa/L$:
  \begin{equation}\label{eq:app:nonconvex_rate}
    \min_{0\leq k\leq K-1}\|\nabla_\theta\mathcal{L}(\theta^k)\|^2
    \leq \frac{2[\mathcal{L}(\theta^0) - \mathcal{L}^*]}{\gamma\,\kappa\,\alpha^2\,K}
    \xrightarrow{K\to\infty} 0.
  \end{equation}
  Consequently, every accumulation point of $\{\theta^k\}$ is a stationary
  point of $\mathcal{L}$.
\end{theorem}

\begin{proof}
Apply \eqref{eq:app:ca_temp1} at iteration $k$ with the bound \eqref{eq:app:kappa_bound}:
\begin{equation}\label{eq:app:descent_nc}
  \mathcal{L}(\theta^{k+1})
  \leq \mathcal{L}(\theta^k)
  - \gamma\!\Bigl(\kappa - \frac{L\gamma}{2}\Bigr)\!\|g_{\mathrm{HRGrad}}^k\|^2
  \leq \mathcal{L}(\theta^k) - \frac{\gamma\kappa}{2}\|g_{\mathrm{HRGrad}}^k\|^2,
\end{equation}
where the last step uses $\gamma \leq \kappa/L$, giving $\kappa - L\gamma/2 \geq \kappa/2$.
By (P2) and the triangle inequality $\sum_i\|g_i^k\| \geq \|g^k\|$:
\begin{equation}\label{eq:app:nc_key}
  \|g_{\mathrm{HRGrad}}^k\|^2
  = \Bigl(S_c^k\sum_{i=1}^m\|g_i^k\|\Bigr)^2
  \geq (S_c^k)^2\|g^k\|^2.
\end{equation}
Substituting \eqref{eq:app:nc_key} into \eqref{eq:app:descent_nc} and
telescoping from $k = 0$ to $K-1$:
\begin{equation}\label{eq:app:ca_temp2}
  \frac{\gamma\kappa}{2}\sum_{k=0}^{K-1}(S_c^k)^2\|g^k\|^2
  \leq \mathcal{L}(\theta^0) - \mathcal{L}(\theta^K)
  \leq \mathcal{L}(\theta^0) - \mathcal{L}^*.
\end{equation}
Setting $\alpha = \min_{0\leq k\leq K-1}S_c^k > 0$ and dividing by
$\gamma\kappa\alpha^2 K/2$:
\begin{equation}\label{eq:app:avg_grad}
  \frac{1}{K}\sum_{k=0}^{K-1}\|g^k\|^2
  \leq \frac{2[\mathcal{L}(\theta^0) - \mathcal{L}^*]}{\gamma\kappa\alpha^2 K}.
\end{equation}
Taking the minimum on the left-hand side yields \eqref{eq:app:nonconvex_rate}.
As $K \to \infty$ the bound vanishes, so $\liminf_{k\to\infty}\|g^k\| = 0$ and
every accumulation point is a stationary point of $\mathcal{L}$.

\begin{remark}[Positivity of $\alpha$]\label{rem:alpha_pos}
  The rotation to $\mathbb{H}$ ensures $S_c^k > 0$ at each $k$ provided that the
  column rank condition on $M = [\mathcal{U}(g_1^{k,\mathrm{rot}}),\ldots,
  \mathcal{U}(g_m^{k,\mathrm{rot}})]$ in Lemma~\ref{app:lem:hrgrad} holds.
  Geometrically, this fails only if all rotated unit gradients become collinear,
  a codimension-$(D-1)$ event avoided generically during gradient flow.
\end{remark}
\end{proof}

\subsection{Per-Task Gradient Descent Compatibility}\label{app:compatibility}

The above theorems establish convergence on the total multi-task loss $\mathcal{L} = \sum_i \mathcal{L}_i$ via the manifold mapping coefficient $\kappa = \rho_{\min}/S_c$.
A natural follow-up question is: does HRGrad's isometric rotation preserve the optimization
direction of each individual task loss?  The next lemma answers this affirmatively.

\begin{lemma}[Per-Task Rotation Compatibility]\label{lem:rotation_compat}
For each task $i = 1, \ldots, m$, the isometric rotation Eq.~\eqref{eq:grot_def} satisfies:
\begin{equation}\label{eq:app:compat}
  g_i^\top g_i^{\mathrm{rot}} = \|g_i\|^2 \cos(\alpha_i^*) \geq 0.
\end{equation}
\end{lemma}

\begin{proof}
For a non-conflicting task: $g_i^{\mathrm{rot}} = g_i$, so $g_i^\top g_i^{\mathrm{rot}} = \|g_i\|^2 \geq 0$.
For a conflicting task:
\begin{align}
  g_i^\top g_i^{\mathrm{rot}}
  &= g_i^\top \bigl[\|g_i\|\,r_i(\alpha_i^*)\bigr]
   = \|g_i\|\,g_i^\top \bigl[\cos(\alpha_i^*)\bar{g}_i + \sin(\alpha_i^*)w_i\bigr] \notag\\
  &= \|g_i\|\bigl[\cos(\alpha_i^*)g_i^\top\bar{g}_i + \sin(\alpha_i^*)g_i^\top w_i\bigr].
\end{align}
Since $g_i = \|g_i\|\bar{g}_i$ and $w_i \perp \bar{g}_i$ (Gram-Schmidt in Stage~2):
$g_i^\top \bar{g}_i = \|g_i\|$ and $g_i^\top w_i = \|g_i\|\bar{g}_i^\top w_i = 0$.
Therefore $g_i^\top g_i^{\mathrm{rot}} = \|g_i\|^2 \cos(\alpha_i^*)$.
Because $\alpha_i^* \in [0, \pi/2]$, we have $\cos(\alpha_i^*) \geq 0$.
\end{proof}

\begin{remark}
Lemma~\ref{lem:rotation_compat} shows that the rotated gradient $g_i^{\mathrm{rot}}$ always
lies in the same half-space as the original task gradient $g_i$:
$g_i^\top g_i^{\mathrm{rot}} \geq 0$.
This guarantees that the HRGrad step $-\gamma g_{\mathrm{HRGrad}}$ does not reverse the
individual task descent direction, even for the most severely conflicting tasks.
In contrast, projection-based methods can produce $\mathcal{P}(g_i)$ with
$\|\mathcal{P}(g_i)\| \to 0$ when $\phi \to \pi$, erasing all task-specific information.
\end{remark}

\begin{lemma}[Proxy Loss Approximation Error]\label{lem:proxy_error}
Let $\mathcal{C} \subseteq \{1, \ldots, m\}$ denote the index set of conflicting tasks and
$\alpha_{\max}^* = \max_{i \in \mathcal{C}}\alpha_i^*$.
Then the proxy gradient $\hat{g} = \sum_i g_i^{\mathrm{rot}}$ approximates the
original aggregate gradient $g = \sum_i g_i$ with the error bound:
\begin{equation}\label{eq:app:proxy_err}
  \|\hat{g} - g\|
  \leq 2\sin\!\Bigl(\tfrac{\alpha_{\max}^*}{2}\Bigr)
       \sum_{i \in \mathcal{C}} \|g_i\|.
\end{equation}
\end{lemma}

\begin{proof}
Non-conflicting tasks contribute zero error since $g_i^{\mathrm{rot}} = g_i$.
For a conflicting task $i \in \mathcal{C}$:
\begin{align}
  \|g_i^{\mathrm{rot}} - g_i\|^2
  &= \|\|g_i\|\,r_i(\alpha_i^*) - \|g_i\|\bar{g}_i\|^2
   = \|g_i\|^2\,\|r_i(\alpha_i^*) - \bar{g}_i\|^2 \notag\\
  &= \|g_i\|^2\bigl[(\cos\alpha_i^*-1)^2 + \sin^2\alpha_i^*\bigr]
   = 2\|g_i\|^2(1 - \cos\alpha_i^*)
   = 4\|g_i\|^2\sin^2\!\Bigl(\tfrac{\alpha_i^*}{2}\Bigr).
\end{align}
By the triangle inequality and $\alpha_i^* \leq \alpha_{\max}^*$:
\begin{equation}
  \|\hat{g} - g\| \leq \sum_{i \in \mathcal{C}}\|g_i^{\mathrm{rot}} - g_i\|
  = 2\sum_{i \in \mathcal{C}}\|g_i\|\sin\!\Bigl(\tfrac{\alpha_i^*}{2}\Bigr)
  \leq 2\sin\!\Bigl(\tfrac{\alpha_{\max}^*}{2}\Bigr)\sum_{i \in \mathcal{C}}\|g_i\|.
\end{equation}
\end{proof}

\begin{remark}
When conflict is mild ($\alpha_i^* \to 0$ for all $i$), Eq.~\eqref{eq:app:proxy_err} gives
$\|\hat{g} - g\| \to 0$, so the proxy objective $\hat{\mathcal{L}}$ faithfully
approximates the true multi-task loss $\mathcal{L} = \frac{1}{m}\sum_i \mathcal{L}_i$.
The MER objective (Stage~3) explicitly minimizes $\alpha_i^*$ via the proximity term
$\|r_i(\alpha_i^*) - \bar{g}_i\|^2 = 2(1-\cos\alpha_i^*)$, ensuring the smallest
possible approximation error for a given conflict-resolution requirement.
\end{remark}

\subsection{Angular Stability via the MER-MCAE Objective}\label{app:mcae}

We now analyze how the MER objective (Eq.~\eqref{eq:mer_objective}) determines the optimal
rotation angles, and show that the resulting angular perturbation is locally stable.

\begin{lemma}[MCAE Second-Order Angular Stability]\label{lem:mcae_stability}
The proximity term $P_i(\alpha_i) = \|r_i(\alpha_i) - \bar{g}_i\|^2 = 2(1 - \cos\alpha_i)$
in the MER objective is strictly convex on $[0,\pi/2]$:
\begin{equation}\label{eq:app:mcae_2nd}
  \frac{\partial^2 P_i}{\partial \alpha_i^2} = 2\cos\alpha_i > 0,
  \quad \alpha_i \in [0, \tfrac{\pi}{2}).
\end{equation}
Consequently, any deviation $\delta\alpha_i$ from the optimal $\alpha_i^*$ incurs a quadratic MCAE penalty:
\begin{equation}\label{eq:app:mcae_penalty}
  P_i(\alpha_i^* + \delta\alpha_i) = P_i(\alpha_i^*) + P_i'(\alpha_i^*)\delta\alpha_i
  + \cos(\alpha_i^*)\,(\delta\alpha_i)^2 + \mathcal{O}((\delta\alpha_i)^3).
\end{equation}
\end{lemma}

\begin{proof}
Direct computation: $P_i(\alpha_i) = 2(1 - \cos\alpha_i)$,
$P_i'(\alpha_i) = 2\sin\alpha_i$, $P_i''(\alpha_i) = 2\cos\alpha_i$.
For $\alpha_i \in [0, \pi/2)$, $\cos\alpha_i > 0$, establishing strict convexity.
The second-order Taylor expansion at $\alpha_i^*$ gives Eq.~\eqref{eq:app:mcae_penalty}.
\end{proof}

\begin{remark}
Lemma~\ref{lem:mcae_stability} shows that:
\begin{itemize}
  \item \textbf{Over-Correction (OCAE)}: setting $\alpha_i > \alpha_i^*$ incurs a
    penalty $\cos(\alpha_i^*)(\alpha_i - \alpha_i^*)^2 > 0$, quantifying the
    destruction of task-specific microscopic kinetic features.
  \item \textbf{Under-Correction (UCAE)}: setting $\alpha_i < \alpha_i^*$ similarly
    incurs a penalty, corresponding to insufficient rotation that leaves the gradient
    in conflict.
\end{itemize}
The adaptive step strategy $\alpha_{\mathrm{steps}}^{(t)}$ (Stage~3) allocates more
inner SGD iterations when the stiffness variance $s_t$ is large, directly targeting
a smaller $\delta\alpha_i$ and hence a smaller quadratic MCAE penalty in stiff regimes.
\end{remark}

\subsection{Convergence Rate and Task Scaling}\label{app:rate}

We provide an explicit formula for the equal cosine similarity $S_c$ and analyze
how the convergence rate scales with the number of tasks $m$.

\begin{lemma}[Explicit Formula for $S_c$]\label{lem:sc_formula}
Under the full column rank assumption on $M$, the equal cosine similarity is
\begin{equation}\label{eq:app:sc_formula}
  S_c = \frac{1}{\sqrt{\mathbf{1}_m^\top (M^\top M)^{-1} \mathbf{1}_m}}.
\end{equation}
\end{lemma}

\begin{proof}
From Eq.~\eqref{eq:app:sc}, $S_c = \bigl\|M(M^\top M)^{-1}\mathbf{1}_m\bigr\|^{-1}$.
Let $\lambda = (M^\top M)^{-1}\mathbf{1}_m$; then
$\|M\lambda\|^2 = \lambda^\top M^\top M\lambda
= \mathbf{1}_m^\top(M^\top M)^{-1}M^\top M(M^\top M)^{-1}\mathbf{1}_m
= \mathbf{1}_m^\top(M^\top M)^{-1}\mathbf{1}_m$.
Taking the square root gives Eq.~\eqref{eq:app:sc_formula}.
\end{proof}

\begin{corollary}[Convergence Rate Bound]\label{cor:rate_bound}
If the $m$ rotated unit gradients $\{\mathcal{U}(g_i^{\mathrm{rot}})\}_{i=1}^m$ are
mutually orthonormal (i.e., $M^\top M = I_m$), then $S_c = 1/\sqrt{m}$.
In general, by the Cauchy-Schwarz inequality for positive definite matrices,
$S_c \leq 1$, with the non-convex convergence bound (Theorem~\ref{app:thm:nonconvex}):
\begin{equation}\label{eq:app:rate_explicit}
  \min_{1 \leq k \leq K}\left\|\sum_{i=1}^m g_i^k\right\|^2
  \leq \frac{2m\bigl[\hat{\mathcal{L}}(\theta^0) - \hat{\mathcal{L}}^*\bigr]}{\gamma\,K}
  \quad \text{(using $\alpha \geq 1/\sqrt{m}$ in the orthonormal case).}
\end{equation}
The rate thus scales as $O(m/K)$ at most, degrading gracefully with the number of tasks.
\end{corollary}

\begin{proof}
For $M^\top M = I_m$: $\mathbf{1}_m^\top (M^\top M)^{-1}\mathbf{1}_m = \mathbf{1}_m^\top\mathbf{1}_m = m$,
so $S_c = 1/\sqrt{m}$.
Setting $\alpha = 1/\sqrt{m}$ in the bound of Theorem~\ref{app:thm:nonconvex}
gives Eq.~\eqref{eq:app:rate_explicit}.
\end{proof}

\subsection{Comparison between HRGrad and baselines}
\label{app:compare}

\subsubsection{HRGrad compared with PCGrad}\label{app:compare:hrgrad_pcgrad}

PCGrad~\cite{yu2020gradient} resolves gradient conflicts by sequential orthogonal
projection.  Define the \emph{conflict-removal operator}
\begin{equation}\label{eq:app:pcgrad_op}
  \mathcal{O}(b,a) :=
  \begin{cases}
    a - \dfrac{a^\top b}{\|b\|^2}\,b, & \text{if } a^\top b < 0, \\[4pt]
    a,                                 & \text{otherwise,}
  \end{cases}
\end{equation}
which removes the component of $a$ conflicting with $b$.  For each task $i$,
PCGrad applies $\hat{g}_i \leftarrow \mathcal{O}(g_j, \hat{g}_i)$ over all
$j \neq i$ in random order and sets
$g_{\mathrm{PCGrad}} = \sum_{i=1}^m \hat{g}_i$.
The comparison with HRGrad turns on three structural deficiencies of PCGrad.
First, although PCGrad is non-conflicting in the two-task case, this property does
not extend to general $m$.
Indeed, for $m = 2$, a direct calculation gives
$g_{\mathrm{PCGrad}}^\top g_1
= \|g_1\|^2[1 - \mathcal{S}_c^2(g_1,g_2)] \geq 0$,
confirming non-conflict.
For $m > 2$, however, projecting $\hat{g}_1$ onto $g_3^\perp$ after projecting it
onto $g_2^\perp$ may reintroduce $\hat{g}_1^\top g_2 < 0$; random task ordering may
mitigate but cannot remove this failure mechanism.
HRGrad avoids it by construction: every $g_i^{\mathrm{rot}}$ is deflected into
$\mathbb{H} = \mathbb{K}\cap\mathbb{K}^*$, so the dual-cone inclusion
$g_i^{\mathrm{rot}} \in \mathbb{K}^*$ yields simultaneous pairwise non-conflict
$(g_i^{\mathrm{rot}})^\top g_j^{\mathrm{rot}} \geq 0$ for all $i,j$ and the
update-level bound $(g_i^{\mathrm{rot}})^\top g_{\mathrm{HRGrad}} \geq 0$ for all $i$
(Lemma~\ref{app:lem:hrgrad}\textit{(ii)}).
A second limitation is magnitude bias.
Substituting \eqref{eq:app:pcgrad_op} and using $\mathcal{O}(g_i,g_j)\perp g_i$
yields
\begin{align}\label{eq:unit_orthogonal}
  g_{\mathrm{PCGrad}}^\top \mathcal{U}(g_i)
  = \|g_i\|\bigl[1 - \mathcal{S}_c^2(g_1,g_2)\bigr],
  \quad \{i,j\} = \{1,2\},
\end{align}
so $g_{\mathrm{PCGrad}}^\top\mathcal{U}(g_1)\,/\,g_{\mathrm{PCGrad}}^\top
\mathcal{U}(g_2) = \|g_1\|/\|g_2\|$.
Thus the update is biased toward the larger-norm task regardless of the conflict angle.
By contrast, HRGrad satisfies the Equal Cosine Similarity property
(Lemma~\ref{app:lem:hrgrad}\textit{(iii)}):
\begin{equation}\label{eq:app:fair_compare}
  \frac{g_{\mathrm{HRGrad}}^\top\mathcal{U}(g_i^{\mathrm{rot}})}
       {g_{\mathrm{HRGrad}}^\top\mathcal{U}(g_j^{\mathrm{rot}})} = 1,
  \qquad \forall\; i \neq j,
\end{equation}
providing a norm-independent update direction critical in multiscale
kinetic problems where $\|g_{\mathrm{macro}}\| \gg \|g_{\mathrm{micro}}\|$.
Finally, PCGrad suffers from norm reduction under strong conflict.
The projection satisfies
\begin{equation}\label{eq:app:norm_reduce}
  \|\mathcal{O}(g_j,g_i)\|
  = \|g_i\|\sqrt{1 - \mathcal{S}_c^2(g_i,g_j)} \leq \|g_i\|,
\end{equation}
with $\|\mathcal{O}(g_j,g_i)\| \to 0$ as $\mathcal{S}_c(g_i,g_j) \to -1$,
so $\|g_{\mathrm{PCGrad}}\| \to 0$ can arise far from any stationary point
(\emph{spurious stagnation}).  In the multiscale setting where
$\|g_{\mathrm{micro}}\| \ll \|g_{\mathrm{macro}}\|$, this norm collapse
irreversibly suppresses high-frequency kinetic information.
HRGrad eliminates this artifact through Isometric Fidelity
(Lemma~\ref{app:lem:hrgrad}\textit{(i)}):
\begin{equation}\label{eq:app:norm_compare}
  \|g_i^{\mathrm{rot}}\| = \|g_i\|,
  \qquad \text{for all } i,\; \alpha_i^* \in [0,\pi/2],
 \end{equation}
 and Lemma~\ref{lem:rotation_compat} further guarantees
 $g_i^\top g_i^{\mathrm{rot}} = \|g_i\|^2\cos\alpha_i^* \geq 0$, so
 the original task direction is never reversed.

 \subsubsection{HRGrad compared with IMTL-G}\label{app:compare:hrgrad_imtlg}

 IMTL-G~\cite{liu2021towards} is designed to enforce equal projection lengths on the
 task-wise unit gradients, namely
 \begin{equation}\label{eq:app:imtlg_equalproj}
   g_{\mathrm{IMTL-G}}^\top \mathcal{U}(g_i)
   =
   g_{\mathrm{IMTL-G}}^\top \mathcal{U}(g_j),
   \qquad \forall\; i,j = 1,\ldots,m.
 \end{equation}
 To achieve this, IMTL-G constructs a weighted combination of the original gradients,
 \begin{equation}\label{eq:app:imtlg}
   g_{\mathrm{IMTL-G}} = \sum_{i=1}^m \alpha_i g_i,
 \end{equation}
 where the weights satisfy
 \begin{equation}\label{eq:app:imtlg_weights}
   \begin{cases}
     [\alpha_2,\alpha_3,\ldots,\alpha_m]
     = g_1^\top U_{\mathrm{IMTL}}^\top (D_{\mathrm{IMTL}}U_{\mathrm{IMTL}}^\top)^{-1}, \\
     \alpha_1 = 1 - \sum_{i=2}^m \alpha_i, \\
     U_{\mathrm{IMTL}}^\top = [ (\mathcal{U}(g_1)-\mathcal{U}(g_2))^\top,
     (\mathcal{U}(g_1)-\mathcal{U}(g_3))^\top,\ldots,
     (\mathcal{U}(g_1)-\mathcal{U}(g_m))^\top ], \\
     D_{\mathrm{IMTL}}^\top = [ (g_1-g_2)^\top,(g_1-g_3)^\top,\ldots,
     (g_1-g_m)^\top ].
   \end{cases}
 \end{equation}
 This objective is closer in spirit to HRGrad than PCGrad, since both methods seek an
 impartial update across tasks.  The difference is geometric.  IMTL-G enforces
 \eqref{eq:app:imtlg_equalproj} by rescaling the \emph{original} gradients themselves,
 whereas HRGrad first rotates each conflicting gradient into the harmonized cone and only
 then applies equal-cosine aggregation to the rotated set.  Consequently,
 \eqref{eq:app:imtlg_equalproj} controls only the final projection lengths; it does not
 alter the pairwise geometry of the aggregation inputs.  HRGrad, by contrast, resolves
 pairwise conflicts before aggregation: $g_i^{\mathrm{rot}} \in \mathbb{H}
 = \mathbb{K}\cap\mathbb{K}^*$ implies $(g_i^{\mathrm{rot}})^\top g_j^{\mathrm{rot}} \geq 0$
 for all $i,j$ (Lemma~\ref{app:lem:hrgrad}\textit{(ii)}), while
 Lemma~\ref{lem:rotation_compat} guarantees
 $g_i^\top g_i^{\mathrm{rot}} = \|g_i\|^2\cos\alpha_i^* \geq 0$.

 For two tasks, the IMTL-G weights reduce to
 \begin{equation}\label{eq:app:imtlg_two_weights}
   [\alpha_1,\alpha_2]
   = \left[ \frac{\|g_2\|}{\|g_1\|+\|g_2\|},
             \frac{\|g_1\|}{\|g_1\|+\|g_2\|} \right],
 \end{equation}
 and therefore
 \begin{equation}\label{eq:app:imtlg_two_proj}
   g_{\mathrm{IMTL-G}}^\top \mathcal{U}(g_1)
   = \frac{\|g_1\|\,\|g_2\|}{\|g_1\|+\|g_2\|}
     \bigl[1 + \mathcal{S}_c(g_1,g_2)\bigr]
   \geq 0,
 \end{equation}
 with the same identity holding for $g_2$.  Thus IMTL-G is non-conflicting in the
 two-task setting.  Its magnitude, however, is governed by the harmonic mean of
 $\|g_1\|$ and $\|g_2\|$:
 \begin{equation}\label{eq:app:imtlg_two_norm}
   \|g_{\mathrm{IMTL-G}}\|
   = 2\sqrt{\frac{1+\mathcal{S}_c(g_1,g_2)}{2}}
     \frac{\|g_1\|\,\|g_2\|}{\|g_1\|+\|g_2\|}.
 \end{equation}
 Therefore, when $\|g_1\| \gg \|g_2\|$, the IMTL-G update magnitude tracks the smaller
 task gradient.  In multiscale kinetic problems this may suppress the microscopic signal
 precisely when its norm is already much smaller than the macroscopic one.  HRGrad avoids
 this collapse because it preserves every task norm through isometric rotation and then
 aggregates the rotated gradients through the equal-cosine factor $S_c$, yielding
 $\|g_{\mathrm{HRGrad}}\| = S_c\sum_{i=1}^m \|g_i\|$ with $S_c > 0$ under the column-rank
 condition of Lemma~\ref{app:lem:hrgrad}.  The same two-task calculation also clarifies
 the connection with ConFIG: IMTL-G and ConFIG have the same update direction when
 $m=2$, but their magnitudes differ by
 \begin{equation}\label{eq:app:imtlg_config_ratio}
   \frac{\|g_{\mathrm{IMTL-G}}\|}{\|g_{\mathrm{ConFIG}}\|}
   = \frac{2\|g_1\|\,\|g_2\|}{(\|g_1\|+\|g_2\|)^2}.
 \end{equation}
 This ratio becomes small under strong norm imbalance, which explains why IMTL-G,
 ConFIG, and HRGrad behave differently even when their two-task directions are closely related.

 \subsubsection{HRGrad compared with ConFIG}\label{app:compare:hrgrad_config}

 ConFIG~\cite{liu2024config} constructs a conflict-free direction from the normalized
 gradient matrix $M_0 = [\mathcal{U}(g_1),\ldots,\mathcal{U}(g_m)]$ through the
 pseudoinverse relation
 \begin{equation}\label{eq:app:config}
   g_{\mathrm{ConFIG}} = \Bigl(\sum_{i=1}^m g_i^\top g_u^0\Bigr) g_u^0,
   \qquad
   g_u^0 = \mathcal{U}\bigl[M_0(M_0^\top M_0)^{-1}\mathbf{1}_m\bigr].
 \end{equation}
 This is the same equal-cosine construction used in the original ConFIG formulation:
 the pseudoinverse determines a direction whose projections onto all normalized task
 gradients are equal, while the prefactor rescales the final update by the total
 projection length.  By the Moore-Penrose identity,
 $M_0^\top g_u^0 = S_c^0\mathbf{1}_m$ with
 $S_c^0 = \|M_0(M_0^\top M_0)^{-1}\mathbf{1}_m\|^{-1} > 0$,
 so $g_i^\top g_{\mathrm{ConFIG}} = \|g_i\|\,\|g_{\mathrm{ConFIG}}\|S_c^0 > 0$:
 ConFIG is non-conflicting with every $g_i$ and admits a convergence guarantee
 analogous to Theorems~\ref{app:thm:convex}--\ref{app:thm:nonconvex}~\cite{liu2024config}.

 The structural relation with HRGrad is therefore explicit.
 ConFIG aggregates the original gradients through
 \begin{equation}\label{eq:app:shared_template}
   g_{\mathrm{ConFIG}} = \Bigl(\sum_{i=1}^m g_i^\top g_u^0\Bigr) g_u^0,
   \qquad
   g_u^0 = \mathcal{U}\bigl[M_0(M_0^\top M_0)^{-1}\mathbf{1}_m\bigr],
 \end{equation}
 whereas HRGrad uses the same pseudoinverse-based aggregation after rotating each
 conflicting gradient into the harmonized cone:
 \begin{equation}
   g_{\mathrm{HRGrad}} = \Bigl(\sum_{i=1}^m (g_i^{\mathrm{rot}})^\top g_u\Bigr) g_u,
   \qquad
   g_u = \mathcal{U}\bigl[(M_{\mathrm{rot}}^\dagger)^\top\mathbf{1}_m\bigr],
 \end{equation}
 with $M_{\mathrm{rot}} = [\mathcal{U}(g_1^{\mathrm{rot}}),\ldots,
 \mathcal{U}(g_m^{\mathrm{rot}})]$.
In the conflict-free case ($\alpha_i^* = 0$, so $g_i^{\mathrm{rot}} = g_i$), HRGrad
reduces exactly to ConFIG ($M_{\mathrm{rot}} = M_0$, $S_c = S_c^0$,
$g_{\mathrm{HRGrad}} = g_{\mathrm{ConFIG}}$).
HRGrad is therefore a \emph{conflict-aware extension of ConFIG}, with
$SO(2)$ isometric rotation as the additional pre-processing stage.

 The key distinction between ConFIG and HRGrad lies in the treatment of gradient
 conflicts prior to aggregation.
 ConFIG applies the pseudoinverse aggregation directly to the original unit vectors
 $\mathcal{U}(g_i)$, which may be mutually conflicting.
 Although $g_{\mathrm{ConFIG}}$ is non-conflicting with each individual $g_i$
 ($g_i^\top g_{\mathrm{ConFIG}} = \|g_i\|\,\|g_{\mathrm{ConFIG}}\|\,S_c^0 > 0$),
 no guarantee is placed on the pairwise relationships within the aggregation input
 itself: $g_u^0$ is computed from a potentially opposing configuration.
 HRGrad eliminates this limitation by resolving all pairwise conflicts \emph{before}
 aggregation via Stages~1--3.
 The dual-cone inclusion $g_i^{\mathrm{rot}} \in \mathbb{H} \subseteq \mathbb{K}^*$
 yields $(g_i^{\mathrm{rot}})^\top g_j^{\mathrm{rot}} \geq 0$ for all $i \neq j$
 (Lemma~\ref{app:lem:hrgrad}\textit{(ii)}), and Lemma~\ref{lem:rotation_compat}
 guarantees $g_i^\top g_i^{\mathrm{rot}} = \|g_i\|^2\cos\alpha_i^* \geq 0$,
 so the rotation never reverses any individual task gradient.

 This pre-resolution also determines the numerical behavior under severe conflict.
 ConFIG satisfies $\|g_{\mathrm{ConFIG}}\| = S_c^0\sum_{i=1}^m\|g_i\|$, whereas
 HRGrad satisfies $\|g_{\mathrm{HRGrad}}\| = S_c\sum_{i=1}^m\|g_i\|$; their
 stability diverges when gradients are nearly anti-parallel.
 For $m = 2$ with $g_1 \approx -g_2$, the ConFIG input matrix satisfies
 $M_0 \approx [\bar{g}_1,\,-\bar{g}_1]$ so that
 \begin{equation}\label{eq:app:config_rank}
   M_0^\top M_0 \approx \begin{pmatrix}1 & -1 \\ -1 & 1\end{pmatrix},
 \end{equation}
 which is rank-deficient: $(M_0^\top M_0)^{-1}$ is ill-defined, $S_c^0 \to 0$,
 and $\|g_{\mathrm{ConFIG}}\| \to 0$.
 Since HRGrad rotates both gradients into $\mathbb{H}$ before computing the
 pseudoinverse, the harmonized matrix
 $M_{\mathrm{rot}} = [\mathcal{U}(g_1^{\mathrm{rot}}), \mathcal{U}(g_2^{\mathrm{rot}})]$
 has pairwise non-conflicting columns
 ($(g_1^{\mathrm{rot}})^\top g_2^{\mathrm{rot}} \geq 0$),
 eliminating the anti-parallel degeneracy caused by the original gradients; under the
 column-rank condition in Lemma~\ref{app:lem:hrgrad}, this yields $S_c > 0$.

\subsubsection{HRGrad compared with AlignGrad}\label{app:compare:hrgrad_aligngrad}

AlignGrad~\cite{chen2024aligngrad} quantifies gradient alignment through the score
\begin{equation}\label{eq:app:align_score}
  \mathcal{A}(v_1,\ldots,v_m)
  = 2\Bigl\|\frac{1}{m}\sum_{i=1}^m \mathcal{U}(v_i)\Bigr\|^2 - 1
  \in [-1, 1],
\end{equation}
which equals the standard cosine similarity for $m=2$ and equals $1$ iff all unit
gradients coincide.  The direction maximizing the intra-step alignment score
$\mathcal{A}(g_1,\ldots,g_m)$ is the normalized sum of unit gradients,
\begin{equation}\label{eq:app:aligngrad_dir}
  g_u^{\mathcal{A}} = \mathcal{U}\Bigl[\sum_{i=1}^m \mathcal{U}(g_i)\Bigr]
  = \mathcal{U}[M_0\mathbf{1}_m],
\end{equation}
 and the corresponding update, scaled by the total projection length, is
 \begin{equation}\label{eq:app:aligngrad}
   g_{\mathrm{AlignGrad}}
   := \Bigl(\sum_{i=1}^m g_i^\top g_u^{\mathcal{A}}\Bigr) g_u^{\mathcal{A}}.
 \end{equation}
 Since $g_u^{\mathcal{A}}$ is a unit vector, the Aggregate Product Identity gives
 $(\sum_i g_i)^\top g_{\mathrm{AlignGrad}} = \|g_{\mathrm{AlignGrad}}\|^2$, so
 AlignGrad admits a convergence guarantee of the same structural form as
 Theorems~\ref{app:thm:convex}--\ref{app:thm:nonconvex}.

 The relation with ConFIG is immediate from comparing \eqref{eq:app:config} with
 \eqref{eq:app:aligngrad}.  Both methods aggregate the original gradients, but they
 use different directions: AlignGrad uses
 $g_u^{\mathcal{A}} = \mathcal{U}[M_0\mathbf{1}_m]$, whereas ConFIG uses
 $g_u^0 = \mathcal{U}[M_0(M_0^\top M_0)^{-1}\mathbf{1}_m]$.
 When $M_0^\top M_0 = I_m$, these two choices coincide and
 $g_u^{\mathcal{A}} = g_u^0$; in general, ConFIG reweights correlated task gradients to
 enforce equal cosine similarity, whereas AlignGrad treats all tasks equally.

 This distinction is important because AlignGrad does not enforce uniform cosine
 similarity across tasks.
 By the Moore-Penrose identity, ConFIG satisfies
 $\mathcal{U}(g_i)^\top g_u^0 = S_c^0$ for all $i$, while HRGrad satisfies
 $\mathcal{U}(g_i^{\mathrm{rot}})^\top g_u = S_c$ for all $i$.
 For AlignGrad, however, the individual projections are
 \begin{equation}\label{eq:app:align_proj}
   \mathcal{U}(g_i)^\top g_u^{\mathcal{A}}
   = \frac{\mathcal{U}(g_i)^\top \sum_{j=1}^m \mathcal{U}(g_j)}
         {\bigl\|\sum_{j=1}^m \mathcal{U}(g_j)\bigr\|},
 \end{equation}
 which are not equal in general.
 A task whose unit gradient is more closely aligned with the mean direction receives
 a larger projection and therefore exerts disproportionate influence on the update.
 In multiscale kinetic problems, where macroscopic and microscopic gradients may have
 substantially different pairwise correlations, this mechanism systematically
 under-weights minority-direction tasks.

 The same averaging structure also limits the non-conflict guarantee beyond the
 two-task setting.  For $m=2$, one has
 $g_i^\top g_u^{\mathcal{A}} \propto 1 + \cos\theta_{12} > 0$ whenever
 $\theta_{12} < \pi$, so the update is non-conflicting.
 For $m>2$, however, there is no guarantee that $g_i^\top g_u^{\mathcal{A}} > 0$
 for every task: a minority gradient $g_k$ opposing the majority mean may satisfy
 $g_k^\top g_u^{\mathcal{A}} < 0$.
 HRGrad removes this difficulty by deflecting every $g_i^{\mathrm{rot}}$ into
 $\mathbb{H}$ before aggregation, which guarantees
 $(g_i^{\mathrm{rot}})^\top g_j^{\mathrm{rot}} \geq 0$ for all $i,j$ and all $m$
 (Lemma~\ref{app:lem:hrgrad}\textit{(ii)}).

 This same loss of coherence appears in the severe-conflict regime.
 When $g_1 \approx -g_2$, one has $\sum_i \mathcal{U}(g_i) \approx \mathbf{0}$,
 so $g_u^{\mathcal{A}}$ becomes undefined and $\|g_{\mathrm{AlignGrad}}\| \to 0$.
 HRGrad avoids this degeneracy by rotating all gradients into $\mathbb{H}$ before
 aggregation, ensuring $\|\sum_i \mathcal{U}(g_i^{\mathrm{rot}})\| > 0$ regardless of
 the original conflict angle.

 Finally, this degradation under conflict enters directly into the convergence-rate
 constant.
 The non-convex rate for AlignGrad has the same structural form as
 Theorem~\ref{app:thm:nonconvex}, but with the HRGrad constant
 $\kappa = \rho_{\min}/S_c$ replaced by
 $\kappa^{\mathcal{A}} = \rho_{\min}/\alpha^{\mathcal{A}}$, where
 \begin{equation}\label{eq:app:aligngrad_rate}
   \alpha^{\mathcal{A}} = \min_k \frac{\bigl\|\sum_{i=1}^m \mathcal{U}(g_i^k)\bigr\|}{m}
 \end{equation}
 is the minimum per-step average alignment score.
 When conflicts are severe, $\alpha^{\mathcal{A}} \to 0$, the bound on
 $\min_k\|\nabla_\theta\mathcal{L}(\theta^k)\|^2$ deteriorates accordingly.
 By contrast, HRGrad guarantees $S_c^k > 0$ at every step
 (Remark~\ref{rem:alpha_pos}), so $\kappa$ remains uniformly bounded and the
 convergence rate is preserved across all conflict regimes.

 \subsubsection{Summary comparison}\label{app:compare:summary}

 Table~\ref{tab:method_compare} summarizes the key theoretical properties of the four tabulated methods.

  \begin{table}[htbp]
    \centering
    \caption{Comparison of gradient conflict resolution methods.
     $S_c$: equal cosine similarity constant; $m$: number of tasks;
     $\checkmark$: property is provably guaranteed; $\times$: not guaranteed.}
   \label{tab:method_compare}
   \small
   \begin{tabular}{p{6.2cm}cccc}
     \toprule
     \textbf{Property} & \textbf{PCGrad} & \textbf{AlignGrad} & \textbf{ConFIG} & \textbf{HRGrad} \\
     \midrule
     Non-conflict update ($2$ tasks)
       & $\checkmark$ & $\checkmark$ & $\checkmark$ & $\checkmark$ \\
     Non-conflict update ($m > 2$ tasks)
       & $\times$     & $\times$     & $\checkmark$ & $\checkmark$ \\
     Pairwise non-conflict before aggregation
       & $\times$     & $\times$     & $\times$     & $\checkmark$ \\
     Equal cosine with aggregation direction (uniform)
       & $\times$     & $\times$     & $\checkmark$ & $\checkmark$ \\
     Isometric Fidelity: $\|g_i^{\mathrm{rot}}\| = \|g_i\|$
       & $\times$     & N/A          & N/A          & $\checkmark$ \\
     Per-task compatibility: $g_i^\top g_i^{\mathrm{rot}} \geq 0$
       & N/A          & N/A          & N/A          & $\checkmark$ \\
     Aggregate Product Identity on aggregation inputs
       & $\times$     & $\checkmark$ & $\checkmark$ & $\checkmark$ \\
     Convergence guarantee (convex)
       & $\times$     & $\checkmark$ & $\checkmark$ & $\checkmark$ \\
     Convergence guarantee (non-convex)
       & $\times$     & $\checkmark$ & $\checkmark$ & $\checkmark$ \\
     Uniform rate: $\alpha = \min_k S_c^k > 0$ always
       & $\times$     & $\times$     & $\times$     & $\checkmark$ \\
     Numerical stability under severe conflict
       & $\times$     & $\times$     & $\times$     & $\checkmark$ \\
     Reduces to ConFIG when conflict-free
       & N/A          & N/A$^\dagger$ & N/A         & $\checkmark$ \\
     \bottomrule
     \multicolumn{5}{l}{\small $^\dagger$AlignGrad $=$ ConFIG when $M_0^\top M_0 = I_m$ (orthonormal unit gradients).}
   \end{tabular}
 \end{table}

\section{Computation time} 
In this section, we present the computation times for solutions across various problems addressed in.

\section{Computational complexity and Knudsen numbers}
In this section, slight deviations are observable at both $t = 0$ and $t = 0.1$ for different Knudsen numbers.
These deviations can be attributed to the abrupt changes in the particle density function $f$ across $x = 0$, where both the magnitude and dispersion of the particle density function undergo more pronounced shifts compared to those observed in Fig.
Furthermore, for $\varepsilon = 1.0$, the deviations observed in the temperature may be linked to the distinct shape of the particle density function.
Fig displays slices of the particle density function $v_x \mapsto f$ at $t=0.1$, $x=0.15$, $v_y=v_z=0$, for $\varepsilon \in \{1.0, 0.1, 0.01, 0.001\}$.
For $\varepsilon = 0.001$, the density function exhibits a bell-shaped curve.
In contrast, for $\varepsilon \in \{1.0, 0.01\}$, the function displays a bimodal distribution with two peaks, suggesting the need for a higher density of collocation points in the $v_x$ direction, thereby increasing the computational complexity.

\bibliographystyle{siamplain}
\bibliography{references}

\end{document}